\documentclass{article}

\usepackage[utf8]{inputenc} 
\usepackage[T1]{fontenc}    
\usepackage{hyperref}       
\usepackage{url}            
\usepackage{booktabs}       
\usepackage{amsmath, amssymb, amsthm, amsfonts}       
\usepackage{nicefrac, xfrac}       
\usepackage{microtype}      
\usepackage{xcolor}         
\usepackage{bm}
\usepackage{enumitem}
\usepackage{graphicx}
\usepackage{algorithm}
\usepackage{algorithmic}
\usepackage{natbib}[numbers, sort&compress]
\usepackage{geometry}
\usepackage{authblk,textcomp}
\usepackage{wrapfig}
\usepackage{caption}
\usepackage[cal=euler]{mathalfa}
\usepackage{libertine}
\usepackage{mathtools}

\makeatletter
\newcommand{\otherlabel}[2]{\protected@edef\@currentlabel{#2}\label{#1}}
\makeatother
\usepackage{parskip}

\usepackage[export]{adjustbox}
\newtheorem{theorem}{Theorem}
\newtheorem{lemma}{Lemma}
\newtheorem{proposition}{Proposition}
\newtheorem{definition}{Definition}

\newtheorem{assumption}{Assumption}
\newtheorem{remark}{Remark}

\usepackage{amsmath}
\usepackage{amssymb}
\usepackage{mathtools}
\usepackage{amsthm}
\usepackage{bbm}

\usepackage[capitalize,noabbrev]{cleveref}

\usepackage{commath}

\renewcommand{\vec}[1]{\boldsymbol{#1}}

\DeclareMathOperator{\polylog}{polylog}

  \providecommand{\R}{\mathbb{R}} 

  \makeatletter
  \def\sign{\@ifnextchar*{\@sgnargscaled}{\@ifnextchar[{\sgnargscaleas}{\@ifnextchar{\bgroup}{\@sgnarg}{\sgn} }}}
  \def\@sgnarg#1{\sgn\rbr{#1}}
  \def\@sgnargscaled#1{\sgn\rbr*{#1}}
  \def\@sgnargscaleas[#1]#2{\sgn\rbr[#1]{#2}}
  \makeatother

  \providecommand{\cE}{\mathcal{E}}
  \providecommand{\cF}{\mathcal{F}}

  \providecommand{\cN}{\mathcal{N}}


\newcommand{\speedup}[1]{{\color{gray}(\ifdim #1 pt > 0.3pt #1\else $< #1$\fi{}$\times$)}}

\makeatletter
\newsavebox{\@brx}
\newcommand{\llangle}[1][]{\savebox{\@brx}{\(\m@th{#1\langle}\)}%
  \mathopen{\copy\@brx\mkern2mu\kern-0.9\wd\@brx\usebox{\@brx}}}
\newcommand{\rrangle}[1][]{\savebox{\@brx}{\(\m@th{#1\rangle}\)}%
  \mathclose{\copy\@brx\mkern2mu\kern-0.9\wd\@brx\usebox{\@brx}}}
\makeatother

\DeclareMathOperator{\Span}{Span}

\providecommand{\norm}[1]{\left\lVert#1\right\rVert}

  \providecommand{\R}{\mathbb{R}} %

  \providecommand{\Ec}[2]{{\mathbb E}_{#2}\left[#1\right] }%

  \providecommand{\erf}{\operatorname{erf}}

  \providecommand{\cE}{\mathcal{E}}
  \providecommand{\cF}{\mathcal{F}}

  \providecommand{\cN}{\mathcal{N}}

\newcommand{\dE}{\mathbb{E}}

\newcommand{\dP}{\mathbb{P}}

  \usepackage{bm}

\providecommand{\mycomment}[3]{\todo[caption={},size=footnotesize,color=#1!20]{\textbf{#2: }#3}}%
\providecommand{\inlinecomment}[3]{%
  {\color{#1}#2: #3}}%
\newcommand\commenter[2]%
{%
  \expandafter\newcommand\csname i#1\endcsname[1]{\inlinecomment{#2}{#1}{##1}}
  \expandafter\newcommand\csname #1\endcsname[1]{\mycomment{#2}{#1}{##1}}
}

  \definecolor{mydarkblue}{rgb}{0,0.08,0.45}
  \usepackage{hyperref}
  \hypersetup{ %
    pdftitle={},
    pdfauthor={},
    pdfsubject={},
    pdfkeywords={},
    pdfborder=0 0 0,
    pdfpagemode=UseNone,
    colorlinks=true,
    linkcolor=mydarkblue,
    citecolor=mydarkblue,
    filecolor=mydarkblue,
    urlcolor=mydarkblue,
    pdfview=FitH}

\usepackage{url}

\definecolor{allowedcolor}{RGB}{21,115,17}
\definecolor{noyhatcolor}{RGB}{254,155,30}
\definecolor{optimalcolor}{RGB}{38,254,254}
\definecolor{notallowedcolor}{RGB}{255,0,0}
\definecolor{nodynamicscolor}{RGB}{192,150,192}

\renewcommand{\vec}[1]{\bm{#1}}
\newcommand\Et[1]{\mathbb{E}_t\left[#1\right]}

\hypersetup{pdfauthor={IdePHICS},pdftitle={TimeBatchsizeTradeoff},%
            colorlinks, linktocpage=true, pdfstartpage=1, pdfstartview=FitV,%
    breaklinks=true, pdfpagemode=UseNone, pageanchor=true, pdfpagemode=UseOutlines,%
    plainpages=false, bookmarksnumbered, bookmarksopen=true, bookmarksopenlevel=1,%
    hypertexnames=true, pdfhighlight=/O,%
    urlcolor=orange, linkcolor=blue, citecolor=blue
        }

\title{Online Learning and Information Exponents: On The Importance of Batch size, and Time~/~Complexity Tradeoffs}

\author[1]{Luca Arnaboldi}
\author[1,3]{Yatin Dandi}
\author[1]{Florent Krzakala}
\author[2]{Bruno Loureiro} 
\author[1]{Luca Pesce}
\author[1]{Ludovic Stephan}

\affil[1]{\small Ecole Polytechnique F\'{e}d\'{e}rale de Lausanne, 
Information, Learning and Physics lab. CH-1015 Lausanne, Switzerland.}
\affil[2]{\small Département d’Informatique, École Normale Supérieure - PSL \& CNRS. 45 rue d’Ulm, F-75230 Paris cedex 05, France.}
\affil[3]{\small 
Ecole Polytechnique F\'{e}d\'{e}rale de Lausanne,
Statistical Physics of Computation Laboratory. CH-1015 Lausanne, Switzerland.}

\date{}
\begin{document}

\maketitle

\begin{abstract} 
We study the impact of the batch size $n_b$ on the iteration time $T$ of training two-layer neural networks with one-pass stochastic gradient descent (SGD) on multi-index target functions of isotropic covariates. We characterize the optimal batch size minimizing the iteration time as a function of the hardness of the target, as characterized by the information exponents.
We show that performing gradient updates with large batches $n_b \lesssim d^{\sfrac{\ell}{2}}$ minimizes the training time without changing the total sample complexity, where $\ell$ is the information exponent of the target to be learned \citep{arous2021online} and $d$ is the input dimension. However, larger batch sizes than $n_b \gg d^{\sfrac{\ell}{2}}$ are detrimental for improving the time complexity of SGD. We provably overcome this fundamental limitation via a different training protocol, \textit{Correlation loss SGD}, which suppresses the auto-correlation terms in the loss function. We show that one can track the training progress by a system of low-dimensional ordinary differential equations (ODEs). Finally, we validate our theoretical results with numerical experiments.
\end{abstract}
\section{Introduction}
\label{sec:main:introduction}

Descent-based algorithms, such as Stochastic Gradient Descent (SGD) and its variations, are the backbone of contemporary machine learning. Their simplicity in implementation, efficiency in operation, and notably effective performance in practice highlight their importance. A mathematical understanding of SGD's effectiveness remains a key focus in the field. Recent progress has been particularly noteworthy in the realm of shallow neural networks. A sequence of works demonstrated that optimizing large width two-layer neural networks can be mapped into a convex optimization problem over the space of probability measures of weights, the so-called mean-field analysis \citep{mei2018mean,chizat2018global,rotskoff2018trainability,sirignano2020mean}. 
Following this breakthrough, a large part of the theoretical effort has shifted to describing what class of functions can be efficiently learned by SGD, i.e. time and computational complexities required to learn a given class of functions. This has been, in particular, 
thoroughly analyzed in a series of recent works focusing on isotropic distributions (e.g. Gaussian, spherical or in the hypercube) and targets depending only on a few relevant directions (a.k.a. \emph{multi-index models}). A key result from this literature is that the time complexity of SGD scales with the covariates dimension according to the so-called \emph{information exponent} \citep{arous2021online} for single-index and \emph{leap complexity} \citep{abbe2021staircase,abbe2023sgd} for multi-index targets, sparking increasing interest from the theoretical machine learning community over the last few months \citep{damian2022neural,damian2024smoothing,dandi2023twolayer,bietti2023learning,ba2023learning, moniri2023theory, mousavihosseini2023gradientbased,zweig2023symmetric}. \\
Our work follows this thread, focusing instead on the effect of batch size $n_b$, parallelization, and sample-splitting into the overall complexity required to learn a multi-index target. Instead of looking at data one-by-one, as is common in theoretical studies, we investigate the finite $n_b$ problem, and characterize the time/complexity tradeoff when learning with one-pass SGD. 
Our central goal is to paint a complete picture 
of how fast generalized linear models and two-layer neural networks adapt to the features of training data as a function of $n_b$, and the structure of the target function. 

Our analysis sheds light on a fundamental limitation of one-pass (or online) SGD, namely that for batch sizes larger than the input dimension, the dynamics of the training algorithm is dominated by negative feedback terms that do not permit to reduce the time iterations needed to learn the target. Therefore, we provide a rigorous solution to this fundamental limitation of SGD by considering gradient updates on the correlation loss. 
Our approach, drawing inspiration from the \textit{summary statistics} method employed by \cite{saad.solla_1995_line,arous2021online,ben2022high}, concentrates on the overlaps of neurons with the target subspace and their norms. This differs from recent studies, such as those by \cite{abbe2022merged} and \cite{damian2022neural}, which focus on the full gradient vector.



\section{Setting, contributions, and related works}
\label{sec:main:setting}
\begin{figure}[t]
    \centering
    \hspace{-1cm}
    \includegraphics[width=.8\textwidth]{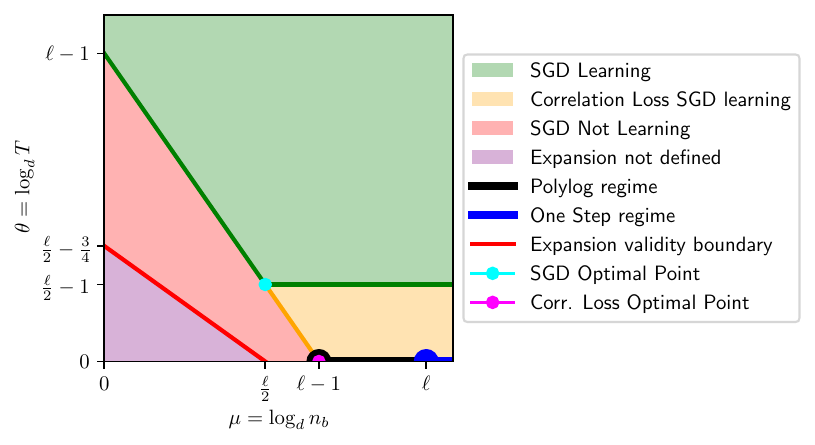}
    \caption{\textbf{Time / Batch size tradeoff for weak recovery:} Phase diagram illustrating different SGD learning regimes as a function of the batch size exponent $\mu=\log_d n_b$ and weak recovery time exponent $\theta=\log_d T$. The analysis is dependent on the target's information exponent $\ell$, this particular plot is valid when \(\ell\ge3\). 
    \textbf{\color{notallowedcolor} Not correlating region:} SGD is not able to achieve weak recovery.
    \textbf{\color{noyhatcolor} Self-interaction regime:} SGD is not able to perform weak recovery, but Correlation loss SGD overcomes this limitation.
    \textbf{\color{allowedcolor} Weak recovery region:} SGD successfully achieve weak recovery.
    Note that it exists an \textbf{\color{optimalcolor} optimal choice} at batch size $n_b =O(d^{\sfrac\ell2})$ that minimizes the number of iterations needed by SGD, and another \textbf{\color{magenta} optimal point} at \(n_b=O(d^{\ell-1})\) for \emph{Correlation Loss}.
    The critical line where \(n_b=\Omega(d^{\ell-1})\) is not addressed by our formal . See details about the other two regions (\textbf{\color{black} Polylog Regime} and \textbf{\color{blue} One-step regime} \citep{dandi2023twolayer}) in Appendix~\ref{app:sec:weakrecoveryglm}.
    } 
    \label{fig:cold_start_pd}
\end{figure}

Consider a two-layer neural network with activation function $\sigma$ and first and second layer weights given by $W\in\mathbb{R}^{p\times d}$ and $\vec{a}\in\mathbb{R}^{p}$ respectively:
\begin{align}
\label{eq:main:newtwork_def}
    f(\vec z) = \frac{1}{p} \sum_{j=1}^p a_j \sigma{( \langle \vec z, \vec w_j \rangle)}\,.
\end{align}
 We are interested in studying the capacity of $f$ to learn from training data $\mathcal{D} = \{(\vec{z}^{\nu}, y^{\nu})_{\nu \in [N]}\in\mathbb{R}^{d+1}\}$. In the following, we work under the following setting.

\paragraph{Data model ---} As hinted in the introduction, we focus on the case the (noisy) labels depend on the covariates only through a projection over a $k$-dimensional subspace:
\begin{align}
\label{eq:main:teacher_def}
     y^\nu = h^\star(W^\star \vec z^\nu)  + \sqrt{\Delta} \xi^\nu, \quad\vec z^\nu \sim \mathcal{N}(0,I_d) 
\end{align}
where $W^\star = \{\vec w^\star_r\}_{r \in [k]}\in \mathbb{R}^{k \times d}$ are the target weights, $h^\star: \mathbb{R}^k \to \mathbb{R}$ is a non-linear activation function, $\xi^\nu \sim \mathcal{N}(0,1)$ is the label noise with variance given by $\Delta\ge0$. We focus on the case where $k = O(1)$ and $d$ is large, i.e. the label only depends on a few directions of a high-dimensional ambient space. The target function $f^{\star}(\vec{z}) = h^{\star}(W^{\star}z)$ is often refereed in the literature as a \emph{multi-index model}. 

Note that the setting above where we assume a generative model for the data and study the capacity of a model to learn is also known as teacher-student model in the literature. We adopt this terminology and refer to $f^\star$ and $f$ as the \textit{teacher} and the \textit{student} functions, respectively. Similarly, we refer to $W^\star$ and $W$ as the \textit{teacher} and \textit{student} weights.

\paragraph{Hardness of the learning task ---} Characterizing what class of targets are efficiently learned by two-layer networks is arguably one of the key question in theoretical machine learning. The pivotal work of \cite{arous2021online} provably describes that for $k=1$, the hardness of the learning task is encoded by a single number, the information exponent $\ell$. More precisely, given the activation $h^\star$ in \eqref{eq:main:teacher_def}, $\ell$ is the lowest degree of the Hermite polynomials $\{\mathrm{He}_j\}_{j \in \mathbb N}$ appearing in the Hermite expansion of $h^\star$. This notion generalizes direction-wise for multi-index models ($k>1$), where $\ell$ is known as leap complexity \citep{abbe2023sgd}. 

\begin{definition}[Information Exponent \citep{arous2021online}]
\label{def:main:information_exponent}
    \begin{align}
        \ell = \min\{j \in \mathbb{N} : \mathbb{E}_{\xi \sim \mathcal{N}(0,1)}\left[ h^\star(\xi) \mathrm{He}_j(\xi) \neq 0\right]\}
    \end{align}
\end{definition}

\paragraph{Training algorithm --- } Given the training data $\mathcal{D}$, we consider the training of $(W,\vec{a})$ under a \emph{sample splitting} scheme: the data is partitioned $\mathcal{D} = \bigcup_{t=1}^{T}\mathcal{D}_{t}$ into $T =  \left \lfloor \sfrac{N}{n_{b}} \right \rfloor$ disjoint batches $\mathcal{D}_{t}$ of size $n_{b}$, which are used, one every iteration, to train the network. We consider a common assumption for the training algorithm that is to decouple the training of the hidden weights $W$ and the second layer weights $\vec a$.
By keeping fixed the second layer weights at initialization $\vec a = \vec a_0$, the hidden layer weights $W$ are estimated using (projected) SGD: 
\begin{align}
    \vec w_{j,t+1} =  \frac{\vec{w}_{j,t} - \gamma \nabla_{\vec w_{j,t}}\ell_t}{\norm{\vec{w}_{j,t} - \gamma \nabla_{\vec w_{j,t}}\ell_t}} \qquad \forall t \in [T], \, \forall j \in [p]
    \label{eq:main:gd_update_weights}
\end{align}
where:
\begin{align} \label{eq:main:projected_sgd}
    \ell_t = \frac{1}{2n_{b}} \sum_{\nu = 1}^{n_b}(y^\nu-f(\vec z^\nu))^2, \qquad \forall t \in [T]
\end{align}
is the empirical risk over a batch of data. Two comments are in place. First, the gradient at each step is computed using the empirical loss given by fresh, previously, unseen samples coming from $\mathcal{D}_{t}$. Each gradient is thus an unbiased estimator of the true gradient, which means that on average this algorithm minimizes the population risk over $W$:
\begin{align}
\label{eq:main:risk}
        \mathcal R = \mathbb{E}_{(\vec z,y)} \left[\frac12 (y-f(\vec z))^2\right]
\end{align}

Second, the spherical projection allow us to focus just on the direction learned by the network, putting aside the effect of the change of the norm of the weights. Note that in equation \eqref{eq:main:gd_update_weights} we have kept the read-out layer $\vec{a}$ fixed. Eventually, the second layer could also be trained with SGD, as the first layer, or even with the Moore-Penrose pseudo-inverse solution; In this paper, however, we consider it fixed and focus on the \textit{feature learning} step, i.e., the recovery of the low-dimensional space spanned by $W^{\star}$.

\paragraph{High-dimensional regime ---} We focus in the high-dimensional regime where $d\to\infty$. Of particular interest is the case where the batch size $n_b$ scales with the dimension $d$. Indeed, in modern machine learning, and in particular in the realm of distributed and federated learning, scenarios with large batches, a single pass, and few iterations often becomes the norm \citep{goyal2017accurate,li2020review} (as for instance when training large language models), further underlining the relevance of this scenario.
%
More precisely, we assume a scaling of the relevant parameters, i.e., learning rate and the batch size, with \(d\), as follows:
\begin{equation}\label{eq:main:gamma_nb_scalings}
\gamma = \gamma_0 d^{-\delta} \quad\text{and}\quad n_b = n_0 d^\mu.
\end{equation}
with \(\mu \ge 0\) and \(\delta\) could possibly be any real value. The exponents $(\delta, \mu)$ characterize the Time / Complexity tradeoff illustrated in the phase diagram (Fig.~\ref{fig:cold_start_pd}). More precisely, the figure shows the time complexity $T = T_0 d^{\theta}$ as a function of the batch size exponent $\mu$. The time exponent $(\theta)$ is linked to the learning rate one $(\delta)$ and the information exponent $(\ell)$, and determining this relation is the main object of analysis of the following sections.  


\paragraph{Weak recovery of the target ---} The central object of our analysis is to characterize the time iterations needed for the SGD dynamics defined in eq.~\eqref{eq:main:projected_sgd} to learn the low-dimensional features $W^\star$. More precisely, we are interested in studying the number of steps to achieve order one correlation with the target weights $W^\star$. We refer to this condition as \textit{weak recovery} of the target subspace, formalized in the following definition.
\begin{definition}[Weak recovery]
\label{def:weak_recovery}
    The \textit{target subspace} $V^\star$ is defined as the span of the rows of the target weights $W^\star$:
    \begin{align}
    \label{eq:main:teacher_subspace_def}
        V^\star = {\rm span}(\vec w^\star_1, \dots, \vec w^\star_k)
    \end{align}
    We define the following weak recovery stopping time for a parameter $\eta \in (0,1)$ independent from $d$:
\begin{equation}
    t^{+}_{\eta} = \operatorname{min}\{t \geq 0: \lVert W W^{\star \top} \lVert_F \geq \eta\}
\end{equation}
\end{definition}

Our key objective is to characterize the largest affordable batch size $n_b$ to achieve weak recovery of the relevant target subspace $V^\star$ while minimizing the training time iterations $T$. Indeed, the updates of one-pass SGD in eq.~\eqref{eq:main:projected_sgd} consist of sums of independent terms that can be parallelized efficiently with decentralized learning protocols.

Our \textbf{main contributions} in this paper are the following:
\begin{itemize}
[itemsep=1pt,leftmargin=1em]
    \item We study how the batch size influences the number of steps required to learn a target function, for different information exponents of the problem. We introduce a schematic phase diagram describing the different learning regimes, see Fig.~\ref{fig:cold_start_pd}. 
    
    \item We show that performing gradient updates with large batch sizes can reduce the training time without changing the total sample complexity to weakly recover the teacher subspace only up to \(n_b \lesssim \Psi(\ell)\) samples per steps, with \(d\) the data dimension and $\ell$ the information exponent of the target. Beyond this limit, larger batch sizes are detrimental for one-pass SGD.
    
     \item We characterize that it is possible to improve over this fundamental limitation of one-pass SGD by using gradient updates on the correlation loss, namely  \textit{Correlation loss SGD}. We provably show that the number of steps needed to weakly correlate with the target with this new training protocol can then be pushed down to \(T=\mathrm{polylog}(d)\) when using batch sizes $n_b = O(d^{\ell -1})$, with $\ell$ the information exponent.
     Additionally, we provide sharp prescription on how to scale the learning rate with batch size and input dimension, in order to achieve the best time-memory tradeoff.
     
    \item We show that the asymptotic training dynamics is described by a system of Ordinary Differential Equations (ODEs) that can be solved exactly. We leverage on the ODE description to characterize the different learning phases of two-layer networks when intialized with non-vanishing initial correlation with the target direction to be learned (warm starts). We also discuss finite \(d\) corrections to the asymptotic dimension-free description. 
    
    \item Finally, we validate and illustrate our theoretical results with numerical experiments.
\end{itemize}

The code to reproduce representative figures are available in the Github repository \href{https://github.com/IdePHICS/batch-size-time-complexity-tradeoffs}{https://github.com/IdePHICS/batch-size-time-complexity-tradeoffs}. We refer to App.~\ref{sec:app:hyperparams} for details on the numerical implementations while the rigorous proofs of the main results are detailed in App.~\ref{sec:app:proofs}. 

\paragraph{Other related works ---} 
The dynamics of Stochastic Gradient Descent (SGD) in two-layer neural networks, particularly when trained on synthetic Gaussian data, have been a topic of interest since the seminal works in the mid-1990s \cite{saad.solla_1995_line,saad1995dynamics,MBiehl_1995,PRiegler_1995}. This area has experienced a resurgence in recent years \cite{tan2023online,goldt2019dynamics,veiga2022phase,arnaboldi2023high, arnaboldi2023escaping,berthier2023learning,arous2021online,paquette2022homogenization, collinswoodfin2023hitting, martin2023impact}.

Many theoretical efforts highlighted the class of functions that are efficiently learned by two layer neural networks. In the context of single-index targets, \cite{arous2021online} introduces the notion of information exponent to quantify the hardness of the learning task. Similarly, for multi-index models,  
\citep{abbe2022merged, abbe2023sgd}, building on their earlier work \citep{abbe2021staircase},
demonstrated how the leap complexity of target functions dictates the amount of training samples needed from two-layer networks in the mean-field limit to learn the target. Note that \cite{abbe2023sgd} also considered the case of $n_b \lesssim O(d)$. A large number of theoretical studies devoted to the understanding of the feature learning regime in two-layer networks often assume an asymptotically vanishing initialization for the second layer weights $\vec a_0$ in eq.~\eqref{eq:main:newtwork_def}, see e.g. \citep{abbe2022merged,berthier2023learning,abbe2023sgd}. Although this assumption is amenable for theoretical characterizations, our analysis provably shows that a careful reasoning on the second layer magnitude is needed to offer a complete portrait of the learning dynamics of SGD. More precisely, we describe a sharp divergence when the batch size $n_b \gg d^{\sfrac{\ell}{2}}$ between the dynamics of SGD when optimizing the MSE loss (vanilla SGD), in contrast to the correlation loss $\Tilde{\ell} = \frac{1}{n_b}\sum_{\nu \in [n_b]} 1-y^{\nu}f(\vec z^{\nu})$ (Correlation loss SGD). The latter training protocol is indeed equivalent to consider an asymptotically vanishing second layer weights $\vec a_0$ at initialization in the optimization routine, e.g. see 
\citep{damian2024smoothing}.

Closer to us, the analysis of the {\it first} gradient descent step with large $n_b$ has been discussed in detail in recent papers \citep{ba2022high, damian2022neural,dandi2023twolayer}. \cite{ba2022high} showed that a single large learning rate gradient step allows to beat kernel methods when the number of training samples is proportional to the input dimension . While their results are limited to single-index target and to a single gradient step, \cite{damian2022neural} further showed that with $n = \omega(d^2)$ samples, two-layer nets can learn multi-index target function with zero first Hermite coefficient ($\ell$=2). \cite{dandi2023twolayer} extended their conditions on the sample complexity to general $\ell\geq 1$, showed this complexity is optimal for single-step learning, and extended the results to higher information exponents. Although motivated from different objectives,  \cite{sclocchi2024different} heuristically sketch a phase diagram for the performance of SGD on realistic datasets as a function of the algorithm's relevant parameters, i.e. batch size and learning rate.

A common assumption in  theoretical studies is to consider
sample-splitting schemes for the training protocol.  At each iteration, the optimization algorithm is ran using a fresh batch of observations of the model, drawn independently of past iterations; this routine has been used extensively in the analysis of iterative algorithms (see e.g. \citep{chandrasekher2021sharp,jain2013low,hardt2014fast,jain2015fast,kwon2019global}).
    
\section{Time / Complexity tradeoffs}
\label{sec:main:cold_start}
In this section, we characterize the intertwined dependence between the batch size and the hardness of the learning task in determining the number of one-pass SGD iterations needed to achieve weak recovery of the teacher subspace as in Definition~\ref{def:weak_recovery}. We offer a detailed picture of the tradeoffs to consider in order to minimize the training iteration time $T$, compactly illustrated in the phase diagram in Fig.~\ref{fig:cold_start_pd}.

\paragraph{Network initialization ---}
We consider random initialization for the hidden layer weights of the network~\eqref{eq:main:newtwork_def}, while the second layer weights are kept fixed:
\begin{align}\label{eq:initial_conditions_unconstrainted}
    \vec w_{j,0}\sim \mathrm{Unif}(\mathbb{S}^{d-1}),\quad a_{j,0} = 1 \qquad j \in [p].
\end{align}
We will refer to this situation as \emph{cold start}, since the initial network correlation with the target directions is vanishing when \(d\to+\infty\). 

\paragraph{Generalized Linear Models ---}
The seminal work of \cite{arous2021online} studies the weak recovery problem for Generalized Linear Models (GLMs), i.e. $p=1$, when learning single-index targets ($k=1$). Starting from randomly initialized networks as defined in~\eqref{eq:main:newtwork_def}, the time iterations needed for one-pass SGD (with one sample per batch) to achieve weak recovery of the target direction respects:
\begin{align}
\label{eq:main:time_complexity_single_index}
    I(\ell) =
    \begin{cases}
        &O(d^{\ell -1} ) \qquad \hspace{0.8em}\text{if $\ell>2$} \\
        &O(d \log{d} ) \qquad \text{if $\ell=2$} \\
        &O(d) \qquad \hspace{2.2em}\text{if $\ell=1$}
    \end{cases}
\end{align}
where $\ell$ is the information exponent of the target $f_\star$. 
\\ As far as weak recovery of the target subspace is concerned, the characterization of multi-index targets follows the same lines of thought just replacing the information exponent by the leap index of the target, e.g. see Definition~\textbf{3} of \cite{dandi2023twolayer} or Definition~\textbf{1} of \cite{abbe2023sgd}. 
Similarly to the information exponent definition, the leap index is the lowest rank of the tensors appearing in the Hermite expansion of the target $f^\star$. 
Therefore, we choose to study in the following the training dynamics for the $p=k=1$ scenario for general batch sizes $n_b$. 
This assumption is useful to provide rigorous guarantees as it largely reduces the complexity of the projected SGD dynamics. However, we argue (supported by numerical simulations in Appendix~\ref{app:multiindex_model}) that the same phenomenology will hold for larger values of $p$ and $k$. 



\subsection{Weak recovery with one-pass SGD}
Consider the gradient descent dynamics defined on the hidden layer weights by eq.~\eqref{eq:main:gd_update_weights}. We focus on the description of the time evolution of the correlation between the network's hidden layer weight and the target direction:
\begin{align}
    m_t= \langle \vec w_t , \vec w^\star \rangle 
\end{align}
Our first main result is to characterize the time to achieve weak recovery of the target direction $\vec w^\star$ as a function of the batch size and the information exponent of the target. We make very weak assumptions on the activation and labeling functions, namely only assuming a sub-polynomial growth:
\begin{assumption}[Polynomial growth]\label{assump:poly_growth}
    The activation function $\sigma$ is differentiable everywhere, except maybe at a finite set of points. Both $\sigma'$ and $f^\star$ are sub-polynomial, i.e. there exists a $k > 0$ and a constant $C$ such that for any $x\in \mathbb R$
    \begin{equation}
        |\sigma'(x)|\leq C(1+x)^k \quad \text{and} \quad |f^\star(x)|\leq C(1+x)^k
    \end{equation}
\end{assumption}

\begin{assumption}[Well-posedness]\label{assump:well_posed}
    Let $(c_k)_{k \geq 0}$ and $(c_k^\star)_{k\geq 0}$ be the Hermite coefficients of $\sigma$ and $h^\star$, respectively. Then $c_\ell \neq 0$, and if $\ell$ is even, then $c_\ell c_\ell^\star > 0$.
\end{assumption}

\begin{assumption}[Initialization]\label{assump:init}
    There exists a $\kappa > 0$ such that $m_0 > \kappa / \sqrt{d}$. Further, if $\ell$ is odd, then $m_0$ is such that
    \[ c_\ell c_\ell^\star m_0 > 0 \]
\end{assumption}
Assumption 2 ensures that the optimization problem is achievable for gradient flow on the population loss $\mathcal{R}$. Indeed, one can show that when $m \approx 0$,
\[ \mathcal{R} = 2(1 - c_\ell c_\ell^\star m^\ell) + o(m^\ell); \]
as a result, if $\ell$ is even and $c_\ell c_\ell^\star < 0$, then $m = 0$ is a local maximum of $\mathcal{R}$ and weak recovery is impossible. When $\ell$ is odd, the point $m=0$ is always a strict saddle, so Assumption \ref{assump:init} that we start on the correct side of the saddle. Under the initialization scheme described by Equation~\ref{eq:initial_conditions_unconstrainted}, the first condition is satisfied with arbitrarily high probability upon decreasing $\kappa$, while the second is a $1/2$-probability event.

We are now in the position to formally state the result:
\begin{theorem} [Projected SGD weak recovery]
\label{thm:main:sgd_weak_recovery}
   Consider the projected SGD algorithm with square loss  (Eqs. \eqref{eq:main:gd_update_weights}, \eqref{eq:main:projected_sgd}), and suppose that Assumptions \ref{assump:poly_growth}-\ref{assump:init} hold. There exist absolute constants $c_\gamma, C_\gamma$ such that if
   \[ \gamma\leq c_\gamma \min\left(1, n_b d^{-\left(\frac\ell2\vee 1\right)} \log(d)^{-C_\gamma}\right), \]
    then for large enough $d$ we have with probability $1-ce^{-c\log(n)^2}$
   \begin{equation}
       t_\eta^+ \leq C\gamma^{-1}d^{\left(\frac{\ell}2 - 1 \right) \vee 0}\log(d).
   \end{equation}
\end{theorem}

\begin{table*}[t]
\fontsize{8.5pt}{8.5pt}
\begin{center}
\begin{tabular}{r| c | c | c |l}
   & SGD & SGD & Correlation loss SGD & One step  \\ 
  &with $n_b \lesssim d^{\sfrac{\ell}{2}}$ &with $ d^{\sfrac{\ell}{2}} <\!\!< n_b \lesssim d^{\max (\ell-1,1) }$ &with $n_b=o\left(d^{\max(\ell-1, 1)}\right)$ &with $n_b \!=\! O(d^\ell)$ \\ 
 \hline  $\ell\!=\!1$& $T \!= O(\sfrac{d}{n_b}), N \!=\! O(d)$ & $T\!=\!O(1), N \!=\! O(d)$ & $T\!=\!O(\sfrac{d}{n_b}), N \!=\! O(d)$   & $T\!=\!1, N \!=\! O(d)$\\ 
 \hline  $\ell\!=\!2$ & $T\!=\! O(\sfrac{d \log{d}}{n_b}), N\!=\!O(d \log{d})$ & $T\!=\!O(\log{d}), N \!=\! O(d \log{d})$ & $T \!=\! O( \sfrac{d\log{d}}{n_b}), N \!=\! O(d \log{d})$ & $T\!=\!1, N \!=\! O(d^2)$ \\
 \hline
 $\ell>2$ & $T \!=\! O(\sfrac{d^{\ell -1}}{n_b}), N \!=\! O(d^{\ell-1})$ & $T=O(d^{\sfrac{\ell}{2} - 1}), N \!=\! O(n_b d^{\sfrac{\ell}{2} - 1})$ & $T\!=\!O(\sfrac{d^{\ell-1}}{n_b}), N \!=\! O(d^{\ell-1})$ & $T\!=\!1, N\!=\! O(d^{\ell})$ \\
 \hline
\end{tabular}
\caption{\textbf{Time~/~Complexity tradeoffs:} Number of iterations $T$ and the total number of samples $N$ needed to achieve weak recovery of the target for different training protocols in high dimensions. \textbf{Left:} One-pass SGD of batch size $n_b = d^{\sfrac{\ell}{2}}$, in this regime the optimal time complexity is obtained rescaling by $n_b$ the result of \cite{arous2021online} for $n_b=1$, i.e. by choosing the optimal learning rate $\gamma = O(n_b d^{-\sfrac{\ell}{2}})$. \textbf{Center-left:} One-pass SGD with batch size $d^{\sfrac{\ell}{2}} <\!\!< n_b \lesssim d^{\max (\ell-1,1)}$, for hard problems ($\ell > 2$) the sample complexity is significantly increased with respect to the $n_b=1$ case up to $N = O(n_bd^{\sfrac{\ell}{2}-1})$. The learning rate cannot be increased proportionally to $n_b$ in this region, fixed to be $\gamma = O(1)$. \textbf{Center-Right:} Correlation loss SGD with $n_b = o(d^{\max(\ell-1, 1)})$, this training protocol overcomes the limitation of SGD when $n_b >\!\!>d^{\sfrac{\ell}{2}}$ and $\ell >2$; the total sample complexity is $N=O(d^{\ell-1})$. The learning rate is fixed again to be proportional to the batch size $\gamma = O(n_bd^{-\sfrac{\ell}{2}})$. \textbf{Right:} The target is weakly recovered with one GD step of $n_b=O(d^\ell)$ batch. The learning rate is chosen as $\gamma = O( d^{\sfrac{(\ell-1)}{2}})$ for the One Step routine \cite{dandi2023twolayer}.
} 
\label{table:weak_recovery}
\end{center}
\end{table*}

 \subsection{Illustration of Theorem \ref{thm:main:sgd_weak_recovery}} The phase diagram in Fig.~\ref{fig:cold_start_pd} exemplifies Theorem~\ref{thm:main:sgd_weak_recovery}. We identify three \textit{learning phases}: SGD learning, Correlation Loss SGD learning, and SGD impossible. These regions are explored by varying the batch size and learning rate exponents $\delta,\mu$. Our theory characterizes the optimal learning rate to achieve the lowest possible time iterations of SGD to weakly recover the target direction $\vec w^\star$ when the batch size respects $n_b = o(d^{\ell -1})$:
 \begin{align}
     \delta^\star(\mu) = 
     \begin{cases}
         &\frac{\ell}{2} - \mu  \qquad \text{if} ~ ~ \mu< \sfrac{\ell}{2} \\
         & 0 \qquad \qquad \text{otherwise} 
     \end{cases}
 \end{align}
 
\textbf{Weak recovery region ---} In the region $n_b \lesssim d^{\sfrac{\ell}{2}}$ there is a net benefit in using larger batch sizes in the SGD optimization. 
This section shows a similar phenomenology to \cite{arous2021online}: if we optimally choose the learning rate exponent $\delta^\star(\mu)$, the number of time iterations needed to weakly recover the teacher direction $\vec w^\star$ is simply $T(n_b) = \sfrac{I(\ell)}{n_b}$, rescaling straightforwardly the time complexity of the $n_b=1$ case in eq.~\eqref{eq:main:time_complexity_single_index}. By considering higher values for the learning rate ($\delta < \delta^\star(\mu)$) SGD is not able to weakly recover the signal as the dynamics is dominated by terms contracting the network~/~target correlation to zero, defining the SGD impossible region. Vice versa, if one takes into account lower learning rates ($\delta > \delta^\star(\mu)$), it is certainly possible to weakly-recover the target, but at a higher time complexity cost. 

\textbf{Self-interaction regime --- } Conversely, the region $d^{\sfrac{\ell}{2}} \ll n_b \lesssim d^{\max (\ell-1,1)}$ does not adhere to the same straightforward paradigm. 
Indeed, standard SGD is not able to achieve weak recovery of the teacher direction using $T(n_b)\! =\! \sfrac{I(\ell)}{n_b}$ time iterations, but a simple modification of it - that we call \emph{Correlation Loss SGD} - is able to. The number of steps needed to weakly recover the target with this new training protocol can then be pushed down to \(T=\mathrm{polylog}(d)\) when using batch sizes $n_b = O(d^{\max (\ell-1,1)})$. We refer to the next section for a detailed analysis of this regime.
 
\textbf{One step regime ---} Recent works have discussed the role of one large learning rate gradient descent step (\textit{giant step}) when training of two-layer networks \citep{ba2022high, damian2022neural, dandi2023twolayer}. More precisely, \cite{dandi2023twolayer} sharply characterizes the section $n_b = \Omega(d^{\ell})$ where it is possible to learn the teacher direction in just one step by setting the learning rate to $\delta_{\rm{giant-step}}(\mu) = \frac{1-\ell}{2}$.

\subsection{The self-interaction regime}
Surprisingly, when the learning rate becomes extensive ($\gamma = \omega(1)$), the usual SGD algorithm struggles to achieve weak recovery. This can be explained by writing the gradient update as
\[ \vec{w}_t + \gamma \vec{g}_t = (1 - \gamma\langle \vec{g}_t, \vec{w}_t \rangle)\vec{w}_t + \vec{g}_t^\bot, \]
where $\vec{g}_t, \vec{g}_t^\bot$ are the gradient at time $t$ and its component orthogonal to $\vec{w}_t$, respectively. As a result, projected gradient descent can be seen as a version of spherical SGD with a random weight decay $\gamma \langle \vec{g}_t, \vec{w}_t \rangle$. When $\gamma = \omega(1)$, this weight decay also becomes of order $\omega(1)$, which leads to very unpredictable behavior of the process $(\vec{w}_t)_{t \geq 0}$.

 In this section, we study a modified version for the training protocol, in which the self-interaction term $\langle \vec{g}_t, \vec{w}_t \rangle$ is much smaller; we will refer to this new algorithm as \emph{Correlation loss SGD} (see e.g. \cite{damian2024smoothing}), as it effectively amounts to gradient updates on the correlation loss:
\begin{align} \label{eq:main:correlation_loss}
\Tilde{\ell} = \frac{1}{n_b}\sum_{\nu \in [n_b]} 1-y^{\nu}f(\vec z^{\nu})
\end{align}

The above-described protocol is equivalent to consider a vanishing initialization scale for the second layer weights $\vec a_0$ of the network~\ref{eq:main:newtwork_def}. Such assumptions are often considered in different theoretical efforts (see e.g. \cite{abbe2022merged,berthier2023learning,abbe2023sgd}). However, Fig.~\ref{fig:cold_start_pd} illustrates that a careful analysis of the initialization scale $\vec a_0$ is needed to paint an exhaustive description of the SGD dynamics. Indeed, considering Correlation loss SGD allows to overcome the limitations highlighted by Theorem~\ref{thm:main:sgd_weak_recovery} for projected SGD. In particular, Correlation loss SGD is able to access the yellow region depicted in Fig.~\ref{fig:cold_start_pd} where the time complexity can be reduced again to $\Tilde{T}(n_b) = \sfrac{I(\ell)}{n_b}$ using the optimal learning rate $\Tilde{\delta}^\star(\mu) = \sfrac{\ell}{2} - \mu$ even for $\mu > \sfrac{\ell}{2}$. This is precisely stated in the following theorem. 
\begin{theorem} [\emph{Correlation Loss SGD} weak recovery]
\label{thm:main:no_yhat_weak_recovery}
   Consider the projected SGD algorithm with correlation loss (eqs. \eqref{eq:main:gd_update_weights}, \eqref{eq:main:correlation_loss}), and suppose that Assumptions \ref{assump:poly_growth}-\ref{assump:init} hold. There exists absolute constants $c_\gamma, C_\gamma$ such that if
   \[ \gamma \leq c_\gamma \log(d)^{-C_\gamma} \min\left( n_b d^{-\left(\frac\ell2\vee 1\right)}, \sqrt{\frac{n_b}d} \right) \]
    
   Then if $d$ is large enough, we have with probability $1-ce^{-c\log(n)^2}$
   \begin{equation}
      t_\eta^+ \leq C \max\left(1, \gamma^{-1}d^{\left(\frac{\ell}2 - 1 \right) \vee 0}\log(d)\right).
   \end{equation}
\end{theorem}
 The derivation of Theorems \ref{thm:main:sgd_weak_recovery} and \ref{thm:main:no_yhat_weak_recovery} generalizes \cite{arous2021online} which studies the $n_b =1$ case. Informally, the result is obtained by analyzing the stability of the equation for the correlation \(m_t\), along with the requirement on the step-size for the suppression of the effects of the noise across time. However, there is a major difficulty introduced by the large stepsize regime: when the gradient updates become larger, the Taylor-inspired bounds used in \cite{arous2021online} become vacuous. We work around this problem by showing that in this regime, there is a \emph{one-step improvement} which jumps directly to meaningful correlation with the target vector. All details can be found in App. \ref{sec:app:proofs}. We provide in Table~\ref{table:weak_recovery} a representative summary of the results in Thms~(\ref{thm:main:sgd_weak_recovery}.~\ref{thm:main:no_yhat_weak_recovery}) characterizing the time/complexity tradeoffs to achieve weak recovery of general single index target $f^\star$.

The theoretical predictions of Thm.~\ref{thm:main:no_yhat_weak_recovery} are evaluated in Fig.~\ref{fig:no_yhat_recovery}. The plot compares the student-teacher weight correlation (\(m_t=\langle \vec w_t, \vec w^\star \rangle\)) achieved by vanilla projected SGD and \emph{Correlation Loss SGD} as a function of time. The teacher activation $h^\star$ is fixed to be the third Hermite polynomial ($\ell = 3$), and the batch size varies, effectively changing the region of the phase diagram considered. In agreement with Theorem~\ref{thm:main:no_yhat_weak_recovery}, the figure shows that \emph{Correlation Loss SGD} is always able to achieve faster weak recovery with respect to SGD. Furthermore, the batch size that can be used with \emph{Correlation Loss SGD} in combination with the optimal learning rate $\Tilde{\delta}^\star(\mu) = \sfrac{\ell}{2} - \mu$ is larger, as presented in the phase diagram of Figure~\ref{fig:cold_start_pd}.

\begin{remark}    Theorem~\ref{thm:main:no_yhat_weak_recovery} does not claim superiority of \emph{Correlation Loss SGD} with respect to plain SGD when trying to fully learn the target, but only for achieving weak-correlation faster (Definition~\ref{def:weak_recovery}). As Figure~\ref{fig:no_yhat_recovery} shows, \emph{Correlation Loss SGD} escapes the initial dynamical plateau faster, but is then limited by a loss function not designed properly to reach the global minimum. In Appendix~\ref{sec:app:hyperparams} we investigate the possibility to combine both the algorithms sketched in Fig.~\ref{fig:no_yhat_recovery}, namely escaping the initialization plateau with Correlation loss SGD and then learn the function with SGD; we refer to this protocol as \textit{Adaptive SGD}. Moreover, as Fig.~\ref{fig:cold_start_pd} and Table~\ref{table:weak_recovery} illustrate, the benefits of using Correlation loss SGD are limited to settings in which $\ell >2$. Indeed, the Self-interaction regime (depicted in yellow in Fig.~\ref{fig:cold_start_pd}) is not present for $\ell \le 2$.
\end{remark}

\begin{figure}[t]
    \centering
    \includegraphics[width=\textwidth]{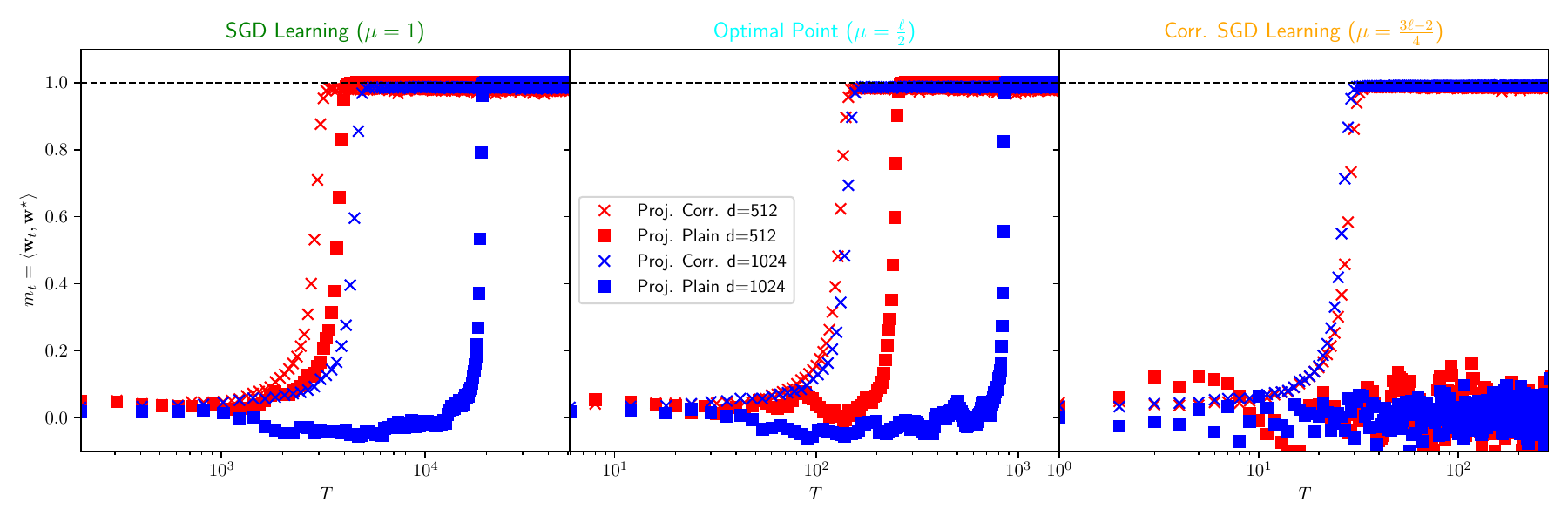}
    \caption{\textbf{{Correlation Loss SGD} weak recovery:} Comparison between the performance of plain SGD and the Correlation Loss SGD, in different regions of the phase diagram, and for different sizes \(d\). The plot shows the test error as a function of the optimization steps. Both the teacher and the student activation functions are fixed to \(\sigma = h^\star =\text{He}_3\), so the information exponent is \(\ell=3\). In all the three plots we vary the value of \(\mu\), while \(\delta = \mu - \sfrac{\ell}{2}\).
    Theorem~\ref{thm:main:no_yhat_weak_recovery} predicts that the \emph{Correlation Loss SGD} weakly recovers the target direction while SGD fails when \(\delta<0\), in accordance to what is shown in the plot. Note that the numbers of steps needed for the target recovery drastically decrease when \(\mu\) becomes large in accordance with Theorems~\ref{thm:main:sgd_weak_recovery},\ref{thm:main:no_yhat_weak_recovery}.
    }
\label{fig:no_yhat_recovery}
\end{figure}
\section{Exact asymptotic description}
\label{sec:main:warm_start}
We now characterize the exact asymptotic description of the dynamics of two-layer networks trained with SGD. In Fig.~\ref{fig:warm_start_pd} we sketch a representative phase diagram as a function of the relevant parameter of the algorithm, i.e. the learning rate and the batch size. The plot identifies different regions of parameters defining the network's learning efficiency.

\paragraph{Sufficient statistics ---}
Our study, like many other efforts \citep{ben2022high,saad.solla_1995_line}, is based on the concentration of the neurons' overlaps with the target subspace and their norms. This approach only requires the knowledge for every optimization step $t \in [T]$ of the above defined overlaps, often referred to as \textit{sufficient statistics}. Let the pre-activations be defined as:
\begin{align}
    \vec \lambda_t = W_t \vec z \qquad  \text{and} \qquad \vec \lambda^\star = W^\star \vec z
\end{align}
Thanks to the Gaussian nature of the data, the pre-activations at any time step $t$ are jointly Gaussian vectors $(\vec{\lambda}_t, {\vec{\lambda}^{\star}})\sim\mathcal{N}(\vec{0}_{p+k}, \Omega_t)$ with covariance
$\Omega_t \in\mathbb{R}^{(p+k)\times (p+k)}
$: 
\begin{equation}
\label{eq:main:covariance_def}
\Omega_t \coloneqq
\begin{pmatrix}
Q_t & M_t\\
{M_{t}^\top} & P
\end{pmatrix}=
\begin{pmatrix}
W_t{W_t}^{\top} & W_t{W^{\star}}^{\top}\\
W^{\star}W_{t}^\top & W^{\star} W^{\star \top}
\end{pmatrix}
\end{equation}
We refer to $M_t,Q_t$ as the \textit{order parameters}.
\begin{figure*}
    \centering
    \includegraphics[width=0.35\textwidth]{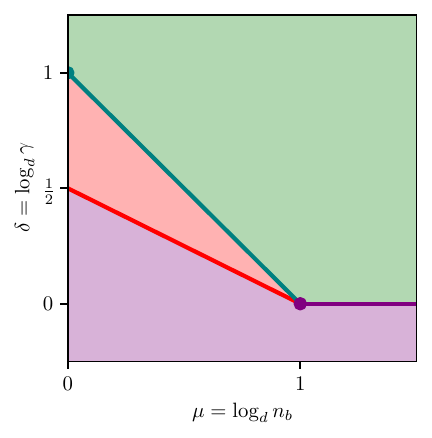}
    \includegraphics[width=0.38\textwidth]{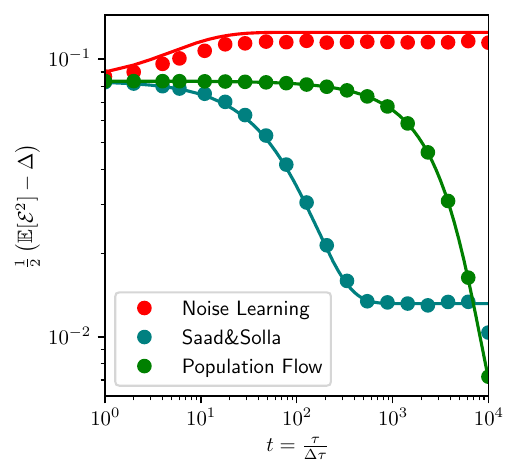}
    \caption{\textbf{Exact asymptotic description:}
        Exact asymptotic characterization of the dynamics of two-layer networks trained with SGD as a function of the batch size ($n_b$) and the learning rate ($\gamma$). \textbf{Left:} Illustration of the different dynamical regimes in a compact phase diagram. \textbf{\color{allowedcolor} Population flow region:} The dynamics is equivalent to population gradient flow. \textbf{\color{notallowedcolor} Noise learning region:} The high-dimensional noise terms dominate the dynamics. 
        \textbf{\color{teal} Saad\&Solla line:}
        The learning dynamics attains a plateau characterized by the noise variance \cite{saad.solla_1995_line}. \textbf{\color{nodynamicscolor} Dynamics not defined:} The deterministic low-dimensional description of the eq.~\eqref{eq:spherical_closed_form_ode} is not valid.
        \textbf{Right:} The figure shows a comparison of numerical simulations (dots) and theoretical prediction (continuous lines) for three instances $(\delta,\mu)$ associated with different learning regimes (identified by the corresponding colors). For both theory and simulations, the test error is plotted as a function of SGD iterations. We consider a matching architectures problem, i.e. \(\ h^\star = \sigma = \erf\) activation, and hidden units \(p=2,k=2\).
    }
    \label{fig:warm_start_pd}
\end{figure*}
\subsection{Closed form equations}
We are now in the position to state our proposition that
provides a set of deterministic ODEs to describe one-pass SGD in high-dimensions.
This portrayal depends ultimately only on the values of the values of the learning rate and the batch size, as quantified by the exponents \((\delta, \mu) \). 
\begin{proposition}
\label{prop:main:exact_asympt}
Consider \(\bar\Omega(t)\) the solution of the system of ordinary differential equations
\begin{equation}
\label{eq:spherical_closed_form_ode}
\begin{split}
    \dod{M_{jr}}{\tau} =& \Psi_{jr}{(\Omega)} - \frac{M_{jr}}{2}\Phi_{jj}{(\Omega)} \\
    \dod{Q_{jl}}{\tau} =& \Phi_{jl}{(\Omega)} - \frac{Q_{jl}}{2}\left(\Phi_{jj}{(\Omega)} + \Phi_{ll}{(\Omega)}\right)
\end{split}
\end{equation}
where we introduced:
\begin{equation}
\label{eq:informal_ODE}
\begin{split}
    \Psi_{jr}(\Omega) =&  \mathbf{1}_{\{\delta\ge0 \bigcap 2\delta+\mu\ge1\}}\frac{\gamma_0}{p}a_j \psi_{jr}  \\
    \Phi_{jl}(\Omega) =&  \mathbf{1}_{\{\delta\ge0 \bigcap 2\delta+\mu\ge1\}}\frac{\gamma_0}{p}\left(a_j^{t}\phi^{\rm{GF}}_{jl}+a_l^{t}\phi^{\rm{GF}}_{lj}\right) \\
        &+\mathbf{1}_{\{\delta+\mu\ge1 \bigcap 2\delta+\mu\le1\}}\frac{\gamma_0^2}{p^2n_0}a_j^{t} a_l^{t} \phi^{\rm{HD}}_{jl} 
\end{split}\end{equation}
and auxiliary integrals bearing expectations over $\mathcal{N}(\vec 0, \Omega)$:
\begin{align*}
    \psi_{jr} =&\mathbb{E}\left[ \sigma'(\lambda_j)\lambda^{\star}_{r} \mathcal{E} \right] \\
    \phi^{\rm{GF}}_{jl}=&\mathbb{E}\left[ \sigma'(\lambda_j)\lambda_{l} \mathcal{E} \right] \\
    \phi^{\rm{HD}}_{jl}=&\mathbb{E}\left[ \sigma'(\lambda_j)\sigma'(\lambda_l)\mathcal{E}^2  \right]  \\
    \mathcal{E} =& g_\star(\vec \lambda^\star) - \frac{1}{p} \sum_{j=1}^p a_j \sigma(\vec \lambda)
    %
\end{align*}
Then, there exists a constant $C$ independent from the input data dimension, such that the discrete stochastic process for the covariance \(\{\Omega_t\}_{t\in\mathbb{N}}\) in eq.~\eqref{eq:main:covariance_def} induced by projected SGD dynamics is approximated by the deterministic covariance matrix \(\bar\Omega(t)\) with precision:
\begin{equation}
\label{eq:main:time_horizon}
    \mathbb{E} \norm{\Omega_t - \bar \Omega(t\Delta \tau)} \le e^{Ct}\sqrt{\Delta \tau}\,
\end{equation}
with
    $\Delta \tau = d^{\max(-\delta,-2\delta+1-\mu)}$.
\end{proposition}
We refer to Appendix~\ref{sec:app:derivation_of_lb_equations} for the informal derivation of the above result. 

In Fig.~\ref{fig:warm_start_pd}~(left) we summarize the results of Prop.~\ref{prop:main:exact_asympt} in a compact phase diagram. 
The following dynamical regimes appear:
\begin{itemize}[noitemsep,leftmargin=1em,wide=0pt]
    \item \textbf{Population Flow}: The dynamics of the sufficient statistics described by a deterministic set of ODEs~\eqref{eq:spherical_closed_form_ode} is equivalent to population gradient flow. 
    \item \textbf{Noise learning}: The dynamic is dominated by high-dimensional noise, and consequently the algorithm does not learn the target;
        the behavior is reflected in the ODEs.
    \item \textbf{Saad\&Solla line}: The ODE description in Prop.~\ref{prop:main:exact_asympt} is equivalent to the pivotal work on 2LNNs \cite{saad.solla_1995_line}. In particular, the original work corresponds to the point \((\delta,\mu)=(1,0)\). The learning dynamics is blocked on a plateau characterized by the noise variance in the labels. 
    \item \textbf{Dynamics not defined}: For a broad range of values of \( (\delta, \mu) \) the SGD dynamics is not effectively described by a set of low-dimensional deterministic ODEs.
\end{itemize}
In the right panel of Figure~\ref{fig:warm_start_pd} we present a numerical investigation of three particular instances of the regimes presented above. The plot shows a comparison of numerical simulations versus the low-dimensional exact asymptotic characterization given in eqs.~\eqref{eq:spherical_closed_form_ode}. The values of the learning rate and the batch size used for SGD training are varied to probe different regions of phase diagram~\ref{fig:warm_start_pd}. 

\begin{remark}
When the target's leap index is $\ell>1$, the dynamic of SGD is dominated by a first extensive search phase to achieve weak recovery of the teacher direction (Thm.~\ref{thm:main:sgd_weak_recovery}). 
Therefore, in order to probe interesting dynamical regimes for general single index teachers, we assume non-vanishing initial correlations of the network's hidden layer weights with the teacher's ones when \(d\to+\infty\).
In App.~\ref{sec:app:hyperparams} we study the tightness of the exponential bound eq.~\eqref{eq:main:time_horizon}; we argue supported by numerical illustrations that (on the practical side) the low dimensional ODE description is valid well beyond the extents of Prop.~\ref{prop:main:exact_asympt}, as already observed by other works \cite{goldt2019dynamics,veiga2022phase} in different context. 
\end{remark}


\paragraph{Non asymptotic corrections ---}
Proposition~\ref{prop:main:exact_asympt} unveils a surprising result for the exact asymptotic description of two-layer networks. Indeed, the ODEs written in eq.~\eqref{eq:spherical_closed_form_ode} coincide with the analogous ones for the single sample per batch case ($n_b=1$), modulo trivial rescaling of the parameters \cite{veiga2022phase}.  However, a careful consideration of the ``intra-batch correlations'' in the gradient is needed for correctly describing the low-dimensional process of the order parameters:
\begin{align}
    \sum_{\nu' = 1, \nu'\neq\nu}^{n_b} \sigma'(\lambda_j^\nu)\sigma'(\lambda_l^{\nu'})\mathcal{E}^\nu\mathcal{E}^{\nu'} \langle \vec z^{\nu}, \vec z^{\nu'} \rangle 
\end{align}
The asymptotic form of this term can be exactly computed to be (using Prop.~\ref{prop:main:exact_asympt} notations):
\begin{align}
\label{eq:main:large_batch_term}
    \phi^{\rm{BC}}_{jl}=&\mathbb{E}\left[\sigma'(\lambda_j)\mathcal E \left(\lambda^\star\right)^\top\right]P^{-1}\mathbb{E}\left[\sigma'(\lambda_l)\mathcal E \lambda^\star\right]+\\
&\mathbb{E}\left[\sigma'(\lambda_j)\mathcal E (\lambda^\bot)^\top\right]\left(Q^\bot\right)^{-1}\mathbb{E}\left[\sigma'(\lambda_l)\mathcal E \lambda^\bot\right]
\end{align}
with $ \vec \lambda^\bot = \vec \lambda - MP^{-1}\vec\lambda^\star\,\,\text{and}\,\,
Q^\bot = Q - M P^{-1}M^\top$.
Although the contribution of the above term is asymptotically vanishing in the ODE description~\eqref{eq:spherical_closed_form_ode} when $d \to \infty$, any theoretical description at finite $d$ will effectively depend on $\phi^{\rm{BC}}_{jl}$. In App.~\ref{sec:app:hyperparams} we provide additional numerical investigation on the importance of~\eqref{eq:main:large_batch_term} and the role of large batch sizes for non-asymptotic corrections to the characterization in Prop.~\ref{prop:main:exact_asympt}. Moreover, we note that taking into account the presence of large batch size is pivotal to illustrate the time~/~complexity tradeoffs for weak recovery of the target subspace, as thoroughly discussed in Section~\ref{sec:main:cold_start}.

\section{Conclusions}
\label{sec:main:conclusion}
In this manuscript, we have explored the intricate relationship between batch size and the efficiency of learning multi-index targets using one-pass SGD on high-dimensional input data. Our findings defies the conventional belief that larger batch sizes invariably lead to better results  and reveals a critical batch size threshold, beyond which the advantages of larger batches wane in terms of computational complexity. Applying gradient updates on the correlation loss one may, however, navigate this limitation. Finally, we also provide a system of low-dimensional ODE to describe the exact asymptotic of the SGD dynamics with arbitrary batch-sizes. Moving forward, we hope this research paves the way for deeper inquiries into the optimization behaviors of learning algorithms, prompting further examination of deeper networks and alternative loss functions.

\section*{Acknowledgement}
 This work was supported by the Swiss National Science Foundation under grant SNSF OperaGOST (grant number 200390) and the Choose France - CNRS AI Rising Talents program.

\bibliographystyle{unsrtnat}
\bibliography{biblio}

\appendix
\onecolumn
\section{Proof of Theorems \ref{thm:main:sgd_weak_recovery} and \ref{thm:main:no_yhat_weak_recovery}}
\label{sec:app:proofs}

\subsection{Preliminaries}

\paragraph{Notations and definitions} We denote by $\polylog x$ any polynomial in $\log x$ with degree $> 1$. Since the case where $n_b = O(1)$ is already covered by the results in \cite{arous2021online}, we shall always assume that $\mu > 0$. The Hermite coefficients of $\sigma$ and $f^\star$ will be denoted by $(c_k)_{k \geq 0}$ and $(c_k^\star)_{k \geq 0}$, respectively. To break the symmetry between $m$ and $-m$ inherent to the problem, we assume without loss of generality that
\[m_0 > 0 \quad \text{and} \quad c_\ell c_\ell^\star > 0.\]

Throughout this section, the update process on $\vec{w}_t$ will be written as
\begin{equation}\label{eq:app:one_step_w}
    \vec{w}_{t+1} = \frac{\vec{w}_t - \gamma\, \vec{g}_t}{\norm{\vec{w}_t - \gamma\, \vec{g}_t}},
\end{equation}
where $\vec{g}_t$ is the gradient at time $t$: $\vec{g}_t = \nabla_{w_{t}} \ell_t$, and $\ell_t$ is the empirical loss at time $t$, that can be either the correlation or the square loss. When considering the update of the process $(\vec{w}_t)_{t \geq 0}$, it will be useful to distinguish between the randomness in $\vec{w}_t$ and the one introduced by the batch drawn at time $t$. To this end, we introduce the filtration $(\cF_t)_{t\geq 0}$ adapted to the process $\vec{w}_t$, and we shall denote by $\dP_t$ (resp. $\dE_t$) the probability (resp. expectation) conditioned on $\cF_t$.

\paragraph{Concentration in Orlicz spaces} We first recall some fact about Orlicz spaces that will be useful for our concentration bounds.
 \begin{definition}
     For any $\alpha \in \R$, let $\psi_\alpha(x) = e^{x^\alpha} - 1$. Let $X$ be a real random variable; the \emph{Orlicz norm} $\norm{X}_{\psi_\alpha}$ is defined as
     \begin{equation}
         \norm{X}_{\psi_\alpha} = \inf \left\{t > 0\::\: \dE\left[ \psi_\alpha\left(\frac{|X|}{t} \right)\right] \leq 1\right\}
     \end{equation}
 \end{definition}
 It can be checked that $\norm{\cdot}_{\psi_\alpha}$ is a well-defined norm on random variables for $\alpha \geq  1$, and can be slightly modified into a norm when $\alpha < 1$; see \cite{ledoux_1991_probability, vaart_1996_weak} for more information. We say that a random variable is sub-gaussian (resp. sub-exponential) if its $\psi_2$ (resp. $\psi_1$) norm is finite. The main use of this definition is the following concentration inequality:
 \begin{lemma}\label{lem:app:orlicz_concentration}
    Let $X$ be a random variable with finite $\psi_\alpha$-norm for some $\alpha > 0$. Then
      \begin{equation}
     \mathbb{P}\left[\left| X - \dE X \right| > t\norm{X}_{\psi_\alpha}\right] \leq 2e^{-t^\alpha}.
    \end{equation}
 \end{lemma}
As a result, any $\psi_\alpha$-norm bound yields exponential convergence tails. Orlicz norms are also well-behaved with respect to products:
 
 \begin{lemma}\label{lem:app:orlicz_submult}
     Let $X$ and $Y$ be two random variables such that $\norm{X}_{\psi_\alpha}$ and $\norm{Y}_{\psi_\beta}$ are finite for some $\alpha, \beta > 0$. Then
    \[\norm{XY}_{\psi_{\lambda}} \leq \norm{X}_{\psi_\alpha} \norm{Y}_{\psi_\beta},\]
    where $\lambda$ is the number satisfying $\frac1\alpha + \frac1\beta = \frac1\lambda$.
\end{lemma}
\begin{proof}
    Assume without loss of generality that $\norm{X}_{\psi_\alpha} = \norm{Y}_{\psi_\beta} = 1$.
    We use the following Young inequality: for any $a, b > 0$, and $p, q$ such that $\frac1p + \frac1q = 1$,
    \[ab \leq \frac{a^p}{p} + \frac{b^{q}}{q}\]
    Applying this inequality to $p= \alpha/\lambda$, $q = \beta/\lambda$, $a = X^\lambda$, $b = Y^\lambda$, we get
    \[ (XY)^\lambda \leq \frac{\lambda X^\alpha}{\alpha} + \frac{\lambda Y^\beta}{\beta}. \]
    Then
    \begin{align*}
        \exp((XY)^\lambda) &\leq \exp\left( \frac{\lambda X^\alpha}{\alpha}\right) + \exp\left(\frac{\lambda Y^\beta}{\beta} \right)\\
        &\leq \frac{\lambda}{\alpha} \exp(X^\alpha) + \frac{\lambda}{\beta} \exp(Y^\beta),
    \end{align*}
    where at the last line we used Young's inequality again with the same $p$ and $q$. The result ensues from taking expectations on both sides, and noticing that $\lambda/\alpha + \lambda/\beta = 1$ by definition.
\end{proof}

 Finally, we shall use the following theorem:
 \begin{theorem}[Theorem 6.2.3 in \cite{ledoux_1991_probability}] \label{thm:app:orlicz_sum}
     Let $X_1, \dots, X_n$ be $n$ independent random variables with zero mean and second moment $\dE X_i^2 = \sigma_i^2$. Then,
     \begin{equation}
         \norm{\sum_{i=1}^n X_i}_{\psi_\alpha} \leq K_\alpha \log(n)^{1/\alpha} \left(\sqrt{\sum_{i=1}^n \sigma_i^2} + \max_{i}\norm{X_i}_{\psi_\alpha} \right)
     \end{equation}
 \end{theorem}

\subsection{Computing the gradient at time $t$}

Throughout this section, the update process on $\vec{w}_t$ will be written as
\begin{equation}
    \vec{w}_{t+1} = \frac{\vec{w}_t - \gamma\, \vec{g}_t}{\norm{\vec{w}_t - \gamma\, \vec{g}_t}},
\end{equation}
where $\vec{g}_t$ is the gradient at time $t$: $\vec{g}_t = \nabla_{w_{t}} \ell_t$, and $\ell_t$ is the empirical loss $t$, that can be either the correlation or the square loss. A direct computation of both gradients implies the following lemma:
\begin{lemma}
Define
\begin{equation}\label{eq:app:def_ghat_gstar}
    \vec{g}_t^\star = \frac{1}{n_b} \sum_{\nu=1}^{n_b} f^\star(\langle \vec{w}^\star, \vec{z}^\nu \rangle) \sigma'(\langle \vec{w}_t, \vec{z}^\nu \rangle) \vec{z}^\nu \quad \text{and} \quad \hat{\vec{g}}_t = \frac{1}{n_b} \sum_{\nu=1}^{n_b} \sigma(\langle \vec{w}_t, \vec{z}^\nu \rangle) \sigma'(\langle \vec{w}_t, \vec{z}^\nu \rangle) \vec{z}^\nu,
\end{equation}
    Then the gradient of the correlation loss $\ell^{\mathrm{corr}}$ is $- \vec{g}_t^\star$, while the gradient of the square loss $\ell^{\mathrm{sq}}$ is $\hat{\vec{g}}_t - \vec{g}_t^\star$.
\end{lemma}

Hence, the main difference between the gradients of the correlation and square loss is a so-called \emph{interaction term} $\hat{\vec{g}}$, that only depends on the learned vector $\vec{w}_t$. Notice that $\vec{g}_t^\star$ is an average of $n_b$ independent variables of the form
\begin{equation}
    \vec{g}_t^{\star\nu} := f^\star(\langle \vec{w}^\star, \vec{z}^\nu \rangle) \sigma'(\langle \vec{w}_t, \vec{z}^\nu \rangle) \vec{z}^\nu,
\end{equation}
and we define $\hat{\vec{g}}_t^\nu$ and $\vec{g}_t^\nu$ in the same way. By Assumption \ref{assump:poly_growth} and Lemma \ref{lem:app:orlicz_submult}, each variable $\vec{g}_t^{\star\nu}$ (resp. $\hat{\vec{g}}_t^\nu, \vec{g}_t^\nu$) has finite $\psi_\alpha$-norm for some $\alpha > 0$, and hence Proposition 2 of \cite{dandi2023twolayer} holds up to $\polylog(n)$ factors. 

We can also compute the conditional expectation of the gradient $\vec{g}_t$:
\begin{lemma}\label{lem:app:exp_g}
    For any $t \geq 0$, 
    \begin{equation}
        \Et{\vec{g}_t^\star} = \phi(m_t) \vec{w}_t^\star +  \psi^{\mathrm{corr}}(m_t) \vec{w}_t 
    \end{equation}
    where $\phi(m_t)$ and $\psi^{\mathrm{corr}}$ are two functions with Taylor expansion
    \begin{equation}
        \phi(m) = \sum_{k=0}^\infty c_{k+1}c_{k+1}^\star m^k \quad \text{and} \quad \psi^{\mathrm{corr}}(m) = \sum_{k=0}^\infty c_{k+2}c_{k}^\star m^k.
    \end{equation}
    Further, we have
    \begin{equation}
        \Et{\hat{\vec{g}}_t} =  c^{\mathrm{sq}} \vec{w}_t \quad \text{with} \quad c^{\mathrm{sq}} = \Ec{z \sigma(z) \sigma'(z)}{z\sim \cN(0, 1)}.
    \end{equation}
\end{lemma}

\begin{proof}
    The expectation of $\vec{g}_t^\star$ is a specialization of Lemma 4 from \cite{dandi2023twolayer} to $r = 1$. By the independence properties of Gaussians, $\Et{\sigma(\langle \vec{w}_t, \vec{z} \rangle) \sigma'(\langle \vec{w}_t, \vec{z} \rangle) \langle\vec{z}, \vec{w}'\rangle} = 0$ as soon as $\vec{w}'$ is orthogonal to $\vec{w}_t$, hence the expectation of $\hat{\vec{g}}_t$ lies along $\vec{w}_t$, and the second result follows.
\end{proof}

In the following, we will denote $\psi^{\mathrm{sq}}(x) = \psi^{\mathrm{corr}}(x) - c^{\mathrm{sq}}$. A generic $\psi$ will be used when no specialization is necessary, so that
\begin{equation}\label{eq:app:exp_g}
    \Et{\vec{g}_t} = -\phi(m_t) \vec{w}^\star - \psi(m_t) \vec{w}_t.
\end{equation}

\subsection{A differential inequality for $m_t$}

 The structure of the proof is similar to the one of \cite{arous2021online}. We define the following stopping times for $\zeta > 0$:
\begin{align}
    t^+_{\zeta} &= \min\{t \geq 0: m_t \geq \zeta\}, & \ t^-_{\zeta} &= \min\{t \geq 0: m_t \leq \zeta \},
\end{align}
and the following $\gamma$-dependent time:
\begin{equation}
    \tilde t_{\gamma, \zeta}^+ = \min\{t \geq 0: \gamma m_t^{\ell-1} \geq \zeta\}
\end{equation}
 Our first goal is to show the following high-probability inequality:
\begin{proposition}\label{prop:app:diff_ineq_m}
Define
\begin{equation}
    t_{\max} = \frac{n_b}{C_{\max} d \log(d)^{C_{\max}}\, \gamma^2},
\end{equation}
for some sufficiently large $C_{\max} > 0$. Then, for a sufficiently small choice of $c_\gamma$:
\begin{enumerate}
    \item If we are using the square loss, and $\gamma \leq c_\gamma (n_b d^{-\ell/2} \wedge 1)$, there exists $c, \eta > 0$ such that
    \begin{equation}
        \mathbb{P}\left( m_t \geq \frac{3}{4} m_0 + c \gamma \sum_{s = 0}^{t-1} m_s^{\ell-1} \quad \forall t \leq t^+_\eta \wedge t_{\max} \right) \geq 1 - c e^{-c\log(n)^2}.
\end{equation}
    \item If we are using the correlation loss, and $\gamma \leq c_\gamma n_b d^{-\ell/2}$, there exist $c, \varepsilon > 0$ such that
    \begin{equation}
            \mathbb{P}\left( m_t \geq \frac{3}{4} m_0 + c \gamma \sum_{s = 0}^{t-1} m_s^{\ell-1} \quad \forall t \leq t^+_\eta \wedge \tilde t_{\gamma, \varepsilon}^+ \wedge t_{\max} \right) \geq 1 - c e^{-c\log(n)^2}.
    \end{equation}
\end{enumerate}
\end{proposition}

The rest of this section is devoted to show Proposition \ref{prop:app:diff_ineq_m}. We define the following ``good'' event at time $t$:
\[ \cE_t \coloneqq \left\{ |\langle \vec{g}_t, \vec{w}_t \rangle | \leq \frac1{2\gamma} \right\} \]

\paragraph{A (almost) deterministic update inequality}

We first expand the projection step to obtain a difference inequality for the process $m_t$. We write 
\begin{equation}\label{eq:app:grad_decomp}
    \vec{g}_t = \langle \vec{g_t}, \vec{w_t} \rangle \vec{w}_t + \vec{g_t}^\bot,
\end{equation}  
where $\vec{g}_t^\bot$ is orthogonal to $\vec{w_t}$. Similarly to Equation~\ref{eq:app:exp_g}, we can compute the expectation of $\vec{g}_t^\bot$:
\begin{equation}\label{eq:app:exp_gbot}
    \Et{\vec{g}_t} = -\phi(m_t)(\vec{w}_t^\star - m_t \vec{w}_t)
\end{equation}
\begin{lemma}\label{lem:app:difference_m_determ}
For any $t \geq 0$, there exists a (random) constant $c_t$ such that the following inequality holds:
\begin{equation}
    m_{t+1} \geq m_t - \gamma c_t \langle \vec{w}^\star,\vec{g}_t^\bot \rangle - \frac{\gamma^2 c_t^2 m_t \norm{\vec{g}_t^\bot}^2}{2} - \frac12\gamma^3 c_t^3 |\langle w^\star, \vec{g}^\bot \rangle| \norm{\vec{g}_t^\bot}^2.
\end{equation}

Further, under the event $\cE_t$, we have $1/2 \leq c_t \leq 2$.
\end{lemma}

\begin{proof}
    We use the decomposition of eq. \eqref{eq:app:grad_decomp} and write
    \[ \vec{w}_t - \gamma \vec{g}_t = (1 - \gamma \langle \vec{g_t}, \vec{w_t} \rangle)\vec{w}_t + \gamma \vec{g}_t^\bot \]
    As a result, if we define
    \[ c_t \coloneqq \frac1{1 - \gamma \langle \vec{g_t}, \vec{w_t} \rangle}, \]
    we have
    \[ \vec{w}_{t+1} = \frac{\vec{w}_t - \gamma c_t \vec{g}_t^\bot}{\norm{\vec{w}_t - \gamma c_t \vec{g}_t^\bot}} \]
    since the update equation \eqref{eq:app:one_step_w} is invariant w.r.t scaling. Taking the scalar product of the above with $\vec{w}^\star$, we have
    \begin{equation}\label{eq:app:one_step_m}
        m_{t+1} = \frac{m_t - \gamma c_t \langle \vec{w}^\star,\vec{g}_t^\bot \rangle}{\norm{\vec{w}_t - \gamma c_t \vec{g}_t^\bot}}.
    \end{equation}
    Expanding the norm in the denominator, and using that $\norm{\vec{w}_t} = 1$ and $\langle \vec{w}_t, \vec{g}_t^\bot \rangle = 0$:
    \[ \norm{\vec{w}_t - \gamma\, \vec{g}_t} = \sqrt{1 + \gamma^2 c_t^2 \norm{\vec{g}_t^\bot}^2} \]
    By the convexity inequality $(1+x)^{-1/2} \geq 1 - x/2$, valid for all $x \geq 0$, Equation \eqref{eq:app:one_step_m} becomes
    \begin{align*}
        m_{t+1} &\geq \left(m_t - \gamma c_t \langle \vec{w}^\star,\vec{g}_t^\bot \rangle \right) \left( 1 - \frac{\gamma^2 c_t^2}2 \norm{\vec{g}_t^\bot}^2 \right).
    \end{align*}
    The lemma ensues upon expanding and rearranging the terms.
\end{proof}

The expansion in Lemma \ref{lem:app:difference_m_determ} can be decomposed in two terms: the term linear in $\gamma$ is a noisy \emph{drift term}, that will drive the dynamics, and that we will decompose as a sum of a deterministic process and a martingale. All other terms in $\gamma^2$ or $\gamma^3$ are corrections that we bound with high probability.

\paragraph{The linear term} We first control the term linear in $\gamma$. We can write
\begin{equation}\label{eq:app:martingale_decomp}
     \langle \vec{w}^\star, \vec{g}_t^\bot \rangle =  \langle \vec{w}^\star, \Et{\vec{g}_t^\bot} \rangle + Z_t,
\end{equation}
where $(Z_t)_{t \geq 0}$ is by definition a martingale difference sequence for the filtration $(\cF_t)_{t \geq 0}$. The expectation term is straightforward to compute using \eqref{eq:app:exp_gbot}:
\begin{equation}
    \langle \vec{w}^\star, \Et{\vec{g}_t^\bot} \rangle = - (1 - m_t^2) \phi(m_t). \label{eq:app:drift_term}
\end{equation}

The contribution of the terms $Z_t$ is bounded by the following lemma:
\begin{lemma}\label{lem:app:bound_martingale}
    There exists constants $c, C > 0$ such that with probability $1-ce^{-c\log(n)^2}$,
    \begin{equation}
        \sup_{1 \leq t \leq T} \sum_{s = 1}^t Z_s  \leq \frac{C\log(d)^C \sqrt{T}}{\sqrt{n_b}}
    \end{equation}
\end{lemma}
\begin{proof}
     The martingale increment $Z_t$ is an average of $n_b$ independent terms $Z_t^\nu$, that satisfy $\norm{Z_t^\nu}_{\psi_\alpha} \leq C$ for some $\alpha, C > 0$ by Assumption \ref{assump:poly_growth}. As a result, if we define
     \begin{equation*}
         B_\alpha = \sup_{t \geq 0} \norm{Z_t}_{\psi_\alpha},
     \end{equation*}
     Theorem \ref{thm:app:orlicz_sum} implies that
    \begin{equation*}
        B_\alpha = \frac{\polylog(d)}{\sqrt{n_b}}.
    \end{equation*}
 We now apply Theorem F.1 in \cite{pmlr-v134-li21a} with $z = \log(d)^{\frac{2(\alpha+2)}{\alpha}}  \sqrt{T} B_\alpha$, which yields the exact bound needed. 
\end{proof}

\paragraph{Bounding the corrections} Our next step is to handle the higher-order corrections. We show the following lemma:
\begin{lemma}\label{lem:app:bound_corrections}
    Let $T \geq 0$, and $\eta < 1$. There exists a constant $C > 0$ such that for any $t \leq T$
    \begin{equation}
        \mathbb{P}\left(\norm{\vec{g}_t^\bot}^2 \leq C\left(\phi(m_t)^2(1 - m_t)^2  + \frac {d\log(d)^C} {n_b}\right) \quad \forall t \leq t_\eta^+ \wedge T \right) \geq 1 - cTe^{-c\log(d)^2}
    \end{equation}
\end{lemma}

\begin{proof}
    Fix some $t \in [T]$. We can write
    \[ \vec{g}_t^\bot = a(\vec{w}^\star - m\vec{w}) + \frac1n \sum_{\nu=1}^n f^\star(\langle \vec{w}^\star, \vec{z}^\nu \rangle) \sigma'(\langle \vec{w}_t, \vec{z}^\nu \rangle) \vec{z}^{\nu\bot}\]
    where each $\vec{z}^{\nu\bot}$ is orthogonal to both $\vec{w}$ and $\vec{w}^\star$. From Lemmas 9 and 11 in \cite{dandi2023twolayer}, with probability $1 - ce^{-c\log(d)^2}$,
    \[ \frac1n \sum_{\nu=1}^n f^\star(\langle \vec{w}^\star, \vec{z}^\nu \rangle) \sigma'(\langle \vec{w}_t, \vec{z}^\nu \rangle) \vec{z}^{\nu\bot} \leq C \frac{d\log(d)^C}{n_b}. \]
    Now, we have
    \[ \langle \vec{g}_t^\bot, \vec{w}_t^\star \rangle^2 = a^2(1 - m_t^2)^2 \quad \text{and} \quad \norm{a(\vec{w}^\star - m\vec{w})}^2 = a^2 (1-m_t^2), \]
    hence
    \[ \norm{ \vec{g}_t^\bot}^2 \leq \frac1{1 - m_t^2} \langle \vec{g}_t^\bot, \vec{w}_t^\star \rangle^2 + C \frac{d\log(d)^C}{n_b}.  \]
    It remains to notice that 
    \[ \langle \vec{g}_t^\bot, \vec{w}_t^\star \rangle^2 = (\langle \vec{w}^\star, \Et{\vec{g}_t^\bot} \rangle + Z_t)^2 \leq 2(\langle \vec{w}^\star, \Et{\vec{g}_t^\bot} \rangle^2 + Z_t^2) \leq \phi(m_t)(1 - m_t^2)^2 + O\left(\frac{d}{n_b}\right).\] 
\end{proof}

\paragraph{Putting it all together} We now combine all the previous bounds into a unique proposition.
\begin{proposition}\label{prop:app:diff_ineq_estimates}
    Let $T \geq 0$. There exists constants $c, C > 0$ such that
    \begin{equation}
        \mathbb{P}\left(m_t \geq m_0 + \sum_{s=0}^{t-1} \Phi_{\mathrm{drift}}(m_s) - C\Phi_{\mathrm{noise}}(m_s) - C K(T) \quad \forall t \leq T \quad \big\vert\quad \bigcap_{t\leq T} \cE_t \right) \geq 1 - Te^{-c\log(n)^2}
    \end{equation}
    where $\Phi_{\mathrm{drift}}$ and $\Phi_{\mathrm{noise}}$ are given by
    \begin{align}
        \Phi_{\mathrm{drift}}(m) &= \gamma (1-m^2) \phi(m), \\
        \Phi_{\mathrm{noise}}(m) &=  \gamma^2 m(1 - m^2)\phi(m)^2 + \gamma^2 m \frac{d \log(d)^C}{n_b} +  \gamma^3 (1 - m^2)^{3/2}\phi(m)^3,  \\
        K(T) &= \frac{ \gamma \log(d)^C \sqrt{T}}{\sqrt{n_b}} +  \gamma^3 T\frac{d \log(d)^C}{n_b}.
    \end{align}
\end{proposition}

\begin{proof}
    By summing the inequality of Lemma \ref{lem:app:difference_m_determ} for $0 \leq s \leq t-1$, we get
    \begin{equation}
        m_{t+1} \geq m_0 - \sum_{s=0}^{t-1} \gamma c_s \langle \vec{w}^\star,\vec{g}_s^\bot \rangle - \frac{\gamma^2 c_s^2 m_s \norm{\vec{g}_s^\bot}^2}{2} - \frac12\gamma^3 c_s^3 \norm{\vec{g}_s^\bot}^3.
    \end{equation}
    The linear term is handled using the martingale decomposition \eqref{eq:app:martingale_decomp} combined with the expectation computation of \eqref{eq:app:drift_term} and the bound of Lemma \ref{lem:app:bound_martingale} on the martingale contribution. The terms in $\gamma^2$ and $\gamma^3$ follow from Lemma \ref{lem:app:bound_corrections}, as well as Lemma \ref{lem:app:bound_martingale} with $T = 1$, which implies that
    \[ |\langle \vec{w}^\star, \vec{g}_t^\bot \rangle \leq C \log(n)^C \]
    for some $C > 0$.
    Finally, under the events $\cE_s$ for $s\leq t$, we can replace every occurence of $c_s$ by either $1/2$ or $2$ depending on the sign of the corresponding term.
\end{proof}

\paragraph{Proof of Proposition \ref{prop:app:diff_ineq_m}} The expressions of Lemma \ref{lem:app:exp_g} imply the following expansions near 0:
\begin{align} 
\phi(m) &= c_{\ell}c_{\ell}^\star m^{\ell-1} + O(m^\ell) & \psi^{\mathrm{corr}}(m) &=  O(m^\ell) & \psi^{\mathrm{sq}}(m) = - c^{\mathrm{sq}} + O(m^\ell) \label{eq:app:phi_psi_bounds}
\end{align} 
As a result, there exists an $\eta > 0$ and constants $C, c > 0$ such that for any $m \leq \eta$,
\begin{align*}
    c m^{\ell - 1} &\leq \phi(m) \leq C m^{\ell - 1} & |\psi^{\mathrm{corr}}(m)| &\leq C m^{\ell} & |\psi^{\mathrm{sq}}(m)| \leq C.
\end{align*}
We first lower bound the drift inequality of Proposition \ref{prop:app:diff_ineq_estimates}. Whenever $m_t \geq \eta$, we have
\[ \Phi_{\mathrm{drift}}(m_t) \geq p(\gamma m_t^{\ell - 1}) - C\gamma^2\frac{d}{n_b},  \]
where $p(x) = x - C(x^2 + x^3)$. Define $\varepsilon > 0$ such that $p(x) > x/2$ on $[0, \varepsilon]$, so when $t \leq t_{\eta}^+ \wedge \tilde t_{\gamma,\varepsilon}^+$
\[ \Phi_{\mathrm{drift}}(m_t) \geq c\gamma m_t^{\ell-1} - C \gamma^2 \frac{d}{n_b} m.\]
When $\gamma \leq c_\gamma n_b d^{-\ell/2} \log(d)^{-C_\gamma}$,
\[ \gamma \frac{d \log(d)^C }{n_b} m \leq c_\gamma (\sqrt{d})^{\ell-2} \leq \frac{cm^{\ell-1}}2 \]
when $t \leq t_{\kappa/2\sqrt{d}}^-$, $C_\gamma \geq C$ and $c_\gamma \leq c/2 (\kappa/2)^{\ell-2}$.

Having shown $\Phi_{\mathrm{drift}}(m_t) \geq c m_t^{\ell-1}$, it remains to handle the constant terms in Proposition \ref{prop:app:diff_ineq_estimates}. We can compute directly $K(t_{\max})$, which yields
\[K(t_{\max}) = \frac{\log(d)^{C - C_{\max}/2}}{C_{\max}\sqrt{d}} + C\gamma\log(d)^{C - C_{\max}}\frac d{n_b} \leq \frac{\log(d)^{C - C_{\max}/2}}{C_{\max}\sqrt{d}} + \frac{Cc_\gamma}{\log(d)^{C_{\max}}} d^{-\frac{\ell}{2}} \leq \frac{m_0}{4}, \]
by choosing $c_\gamma$ small enough and $C_{\max}$ large enough.

Finally, we need to show that the events $\cE_t$ occur with high probability. This is covered by the following lemma:
\begin{lemma}
    Let $T \geq 0$. The following bounds hold:
    \begin{itemize}
        \item for the square loss, if $\gamma \leq c_\gamma$ for small enough $c_\gamma$,
        \[ \mathbb{P}\left(\cE_t \text{ holds for all } t \leq t_\eta^+ \wedge T \right) \geq 1 - Te^{-c\log(d)^2}; \]
        \item for the correlation loss, for small enough $\varepsilon$,
        \[ \mathbb{P}\left(\cE_t \text{ holds for all } t \leq t_\eta^+ \wedge \tilde t_{\gamma, \varepsilon}^+ \wedge T \right) \geq 1 - Te^{-c\log(d)^2}. \]
    \end{itemize}
\end{lemma}

\begin{proof}
    We begin with the case of the square loss. From the expression of the gradient expectation in \eqref{eq:app:exp_g}, and the estimates \eqref{eq:app:phi_psi_bounds},
    \[ \Et{\langle \vec{w}_t, \vec{g}_t \rangle} = \psi^{\mathrm{sq}}(m_t) - \phi(m_t)m_t = O(1),\]
    whenever $t \leq t_\eta^+$. Lemma \ref{lem:app:orlicz_concentration} applied to $\langle \vec{w}_t, \vec{g}_t \rangle$ implies that with probability $1 - ce^{-c\log(n)^2}$
    \[ |\langle \vec{w}_t, \vec{g}_t \rangle| \leq |\Et{\langle \vec{w}_t, \vec{g}_t \rangle}| +\frac{C\log(d)^C}{\sqrt{n_b}} = O(1) \]
    whenever $\mu > 0$. As a result, if $\gamma \leq c_\gamma$ for $c_\gamma$ small enough, $\cE_t$ holds. The proof for the correlation loss proceeds identically, noting this time that
    \[ \Et{\langle \vec{w}_t, \vec{g}_t \rangle} = \psi^{\mathrm{corr}}(m_t) - \phi(m_t)m_t = O(1) \]
    whenever $t \leq t_\eta^+ \wedge \tilde t_{\gamma, \varepsilon}^+$.
\end{proof}

\subsection{From the linear regime to one-step recovery}
We now prove Theorems \ref{thm:main:sgd_weak_recovery} and \ref{thm:main:no_yhat_weak_recovery}. We focus on the case of the correlation loss; the square loss is identical apart from the additional $\gamma \leq c_\gamma$ requirement of Proposition \ref{prop:app:diff_ineq_m}. 

We first assume that $n_b = O(d^{\ell-1})$ and
\begin{equation}\label{eq:app:bound_gamma}
    \gamma \leq  c_\gamma \log(d)^{-C_\gamma} \min\left( n_b d^{-\left(\frac\ell2 \vee 1 \right)},  \right) 
\end{equation} 
for constants $c_\gamma, C_\gamma$ to be chosen later. In particular, the condition $\gamma < c_\gamma n_b d^{-\ell/2}$ of Proposition \ref{prop:app:diff_ineq_m} is satisfied.

\paragraph{The linear regime} The first part of the proof proceeds as in \cite{arous2021online}. Define the function
\begin{equation}
    t_\ell(d) = \begin{cases}
        1 &\text{if } \ell = 1 \\
        \log(d)  &\text{if } \ell = 2 \\
        d^{\frac{\ell}2-1}  &\text{if } \ell > 2
    \end{cases}, 
\end{equation}
and $t_{\mathrm{conv}} = \max(1, \gamma^{-1}t_\ell(d))$.
Proposition \ref{prop:app:diff_ineq_m} as well as Section 5 from \cite{arous2021online} implies the following lemma:
\begin{lemma} \label{lem:app:boundongamma}
    There exists a constant $C > 0$ such that if $t_{\max} \geq C t_{\mathrm{conv}}$, then with probability at least $1 - ce^{-c\log(n)^2}$,
    \begin{equation}
        \tilde t_{\gamma, \varepsilon} \wedge t_\eta^+ \leq Ct_{\mathrm{conv}}.
    \end{equation}
\end{lemma}
We therefore only need to check the condition $t_{\max} \geq C t_{\mathrm{conv}}$.
Plugging the expression for $\gamma$ and $t_{\max}$, we get
\begin{align*}
    \frac{\gamma t_{\max}}{t_\ell(d)} &\geq \frac{n_b}{\gamma C_{\max} \log(d)^{C_{\max}+1} \gamma d^{1 + (\ell/2 - 1)\vee 0}} \geq \frac1{c_\gamma C_{\max}} \log(d)^{C_\gamma - C_{\max} - 1} \\
     t_{\max} &\geq \frac{n_b}{C_{\max}\log(d)^{C_{\max}} d \gamma^2} \geq \frac1{c_\gamma C_{\max}} \log(d)^{2C_\gamma - C_{\max} - 1}
\end{align*} 
Whenever $n_b = O(d^{\ell-1})$, by decreasing $c_\gamma$ and increasing $C_\gamma$, for large enough $d$ and any constant $C$ we have
\[\frac{\gamma t_{\max}}{t_\ell(d)} \leq C \quad \text{and} \quad t_{\max} \geq C, \]
as requested in Lemma \ref{lem:app:boundongamma}.

Whenever $\gamma \eta^{\ell-1} \leq \varepsilon$, we have $\tilde t_{\gamma, \varepsilon} \geq t_\eta^+$ and hence the proof of Theorem \ref{thm:main:no_yhat_weak_recovery} is complete. It thus remains to treat the converse case. Note that the latter only happens in the correlation loss case, since we can always choose $c_\gamma$ such that $c_\gamma \eta^{\ell-1} \leq \varepsilon$; the dynamics of square loss SGD are therefore only in the linearized regime.

\paragraph{One-step recovery above $\tilde t_{\gamma, \varepsilon}$} We now treat the case where $\gamma \eta^{\ell-1} \geq \varepsilon$. For simplicity, let $t = \tilde t_{\gamma, \varepsilon}$, then Lemmas \ref{lem:app:bound_corrections} and \ref{lem:app:bound_martingale} imply that with probability $1 - ce^{-c\log(n)^2}$
\begin{align*}
    \norm{\vec{g}_{t}}^2 &\leq C\left((|\phi(m_t)| + |\psi^{\mathrm{corr}}(m_t)|)^2 + \frac{d}{n_b}\right) \leq C \left(m_t^{2\ell-2} + \frac{d}{n_b}\right) \\
    \langle \vec{g}_t, \vec{w}^\star \rangle &= \phi(m_t) + m_t\psi^{\mathrm{corr}}(m_t) + O\left( \frac{\polylog(d)}{\sqrt{n_b}} \right) \geq cm_t^{\ell - 1} + O\left( \frac{\polylog(d)}{\sqrt{n_b}} \right)
\end{align*}
By definition of $t$, we have $1 \leq \varepsilon^{-1}\gamma m_t^{\ell-1}$, and whenever $\gamma \leq c_\gamma\sqrt{n_b/d}$ one has
\[ \frac\gamma{\sqrt{n_b}} \leq c_\gamma d^{-1/2} \quad \text{and} \quad \gamma^2 \frac{d}{n_b} \leq c_\gamma^2.\]
Therefore, for large enough $d$,
\begin{align*} 
m_t + \gamma \langle \vec{g}_t, \vec{w}^\star \rangle &\geq  \gamma \langle \vec{g}_t, \vec{w}^\star \rangle \geq c\gamma m_t^{\ell-1}\\
\norm{\vec{w}_t + \gamma \vec{g}_t} &\leq 1 + \gamma \norm{\vec{g}_t} \leq 1 + C \gamma \left(m_t^{\ell-1} + \frac{d}{n_b}\right) \leq (C + \varepsilon^{-1} + c_\gamma^2\varepsilon^{-1}) \gamma m_t^{\ell-1} 
\end{align*}
But by taking the scalar product of Equation \eqref{eq:main:gd_update_weights} with $\vec{w}^\star$,
\[ m_{t+1} = \frac{m_t + \gamma \langle \vec{g}_t, \vec{w}^\star \rangle}{\norm{\vec{w}_t + \gamma \vec{g}_t}} \geq \frac{c}{C + \varepsilon^{-1} + c_\gamma^2\varepsilon^{-1}} =: \eta'. \]
Theorem \ref{thm:main:no_yhat_weak_recovery} ensues by redefining $\eta = \min(\eta, \eta').$
\section{Informal derivation of Proposition~\ref{prop:main:exact_asympt}} \label{sec:app:derivation_of_lb_equations}
In this appendix we provide an informal derivation of the low-dimensional deterministic expressions describing the dynamics of the sufficient statistics (Prop.~\ref{prop:main:exact_asympt}).  While the formal rigorous characterization should in principle follow directly from \cite{veiga2022phase}, it requires a significant amount of work for full mathematical rigor.

 Let $\mathcal{D}$ be the set of labeled data $\{\vec{z}^{\nu}, y^{\nu}\}_{\nu \in [n]}$, with label generated by:
\begin{align}
    y^\nu = f^\star(W^\star z^\nu) + \sqrt{\Delta} \xi^\nu, 
\end{align}
where \(W^\star \in \mathbb R^{k\times d} \) where we assume \(k =O(1)\). We are implying that \(y^{\nu}\) depends on \(\vec z^\nu \sim \mathcal{N}(0,I_d)\) just throught a low-dimensional representation (linear latent variable). \(\xi^\nu\sim \mathcal{N}{(0,1)}\) is the artificial noise.

We can track the overlap matrix using standard manipulation. 
We introduce the local fields as:
\begin{equation}
    \vec{\lambda}^{\nu}\coloneqq W\vec{z}^{\nu}\in\mathbb{R}^{p}, \quad {\vec{\lambda}^{\star}}^{\nu}\coloneqq W^{\star}\vec{z}^{\nu}\in\mathbb{R}^{k} \qquad \forall \nu \in [n]
\end{equation}
We fit these data using a two-layer neural network. Let the first layer weights be $W \in \mathbb{R}^{p \times d}$, the second layer weights \(\vec a\in\mathbb{R}^p\); the full expression of the network is given by
\[
    f(\vec z) = \frac{1}{p} \sum_{j=1}^p a_j \sigma{(\vec w_j^\top \vec z)},
\]
where \(w_j\) are the rows of \(W\) and \(\sigma\) is the activation function.

Since $\vec{z}^\nu$ is Gaussian and independent from $(W, W^{\star})$, the pre-activations are jointly Gaussian vectors $(\vec{\lambda}^{\nu}, {\vec{\lambda}^{\star}}^{\nu})\sim\mathcal{N}(\vec{0}_{p+k}, \Omega)$ with covariance: 
\begin{equation}
\Omega \coloneqq
\begin{pmatrix}
Q & M\\
{M^{\top}} & P
\end{pmatrix}=
\begin{pmatrix}
W{W}^{\top} & W{W^{\star}}^{\top}\\
W^{\star}W^{\top} & W^{\star} W^{\star T}
\end{pmatrix}
\in\mathbb{R}^{(p+k)\times (p+k)}
\end{equation}
These is the low-dimensional matrix (its dimensions stay finite even when \(d\to+\infty\)) that contains all the information needed for the dynamics.

We are going to train the network with layer-wise SGD without replacement, using at each time step \(t\) a fresh new batch of size $n_b$: 
\[
    \ell_t = \frac{1}{2n_b} \sum_{\nu=1}^{n_b}(y^\nu_t - f(\vec z^\nu_t))^2
\]
A stated in the main text, we are interested in the \emph{representation learning} phase. The gradient of the first layer weights
\[
    \nabla_{\vec w_j} \ell_t = -\frac{1}{pn_b} \sum_{\nu=1}^{n_b} a_j \sigma'(\lambda^\nu_{j,t})\mathcal{E}_t^\nu \vec{z}^\nu_t \qquad \forall j \in [p]
\]
where we defined for convenience the displacement vector 
\begin{equation}
\mathcal{E}_t^{\nu} \coloneqq y^\nu_t-f{(\vec z^\nu_t)}.
\end{equation}
Initially, we focus on plain SGD, without normalizing the weights at every step. Let us take now one gradient step with learning rate $\gamma$:
\begin{align}
    \vec w_{j,t + 1} = \vec w_{j,t } - \gamma  \nabla_{\vec w_j} \ell_t
    \label{eq:gd_update_weight}
 \end{align}
By combining the gradient update equation with the definitions of the matrices $(W,W^*)$ we obtain the following dynamics:
\begin{equation} \label{eq:ap:full_process}
\begin{aligned}
M_{jr,t+1} - M_{jr,t} =&  \frac{\gamma}{pn_b} a_j \sum_{\nu = 1}^{n_b} \sigma'(\lambda_{j,t}^\nu)\lambda^{\star}_{r} \mathcal{E}_t^{\nu} \\
Q_{jl,t+1} - Q_{jl,t} =& \frac{\gamma}{pn_b} \sum_{\nu = 1}^{n_b} \left(a_j\sigma'(\lambda_{j,t}^\nu)\lambda_{l,t}^{\nu}+a_l\sigma'(\lambda_{l,t}^\nu)\lambda_{j,t}^{\nu}\right) \mathcal{E}_t^\nu \\
    &+\frac{\gamma^2}{p^2n_b^2}a_ja_l\sum_{\nu = 1}^{n_b} \sum_{\nu' = 1}^{n_b} \sigma'(\lambda_{j,t}^\nu)\sigma'(\lambda_{l,t}^{\nu'})\mathcal{E}_t^\nu\mathcal{E}_t^{\nu'} \vec z^{\nu\top}_t \vec z_t^{\nu'}
\end{aligned}
\end{equation}
These equations introduce a discrete stochastic process \(\{\Omega_t\}_{t\in \mathbb{N}}\) that describes the dynamics in alow-dimensional way.
We also introduce the population loss as
\begin{equation}
    \mathcal R_t = \frac12 \mathbb{E}_{\Omega_t}\left[\mathcal E^2\right],
\end{equation}
since it is the quantity telling us the performace of our trained network.

\paragraph{Handling the intra-batch correlations}
Up to now, we have followed the same derivation as the original Saad\&Solla equations \cite{saad.solla_1995_line}, apart from the effective learning rate scaling. Using larger batches introduces some extra correlations terms that have to be taken into account. Let's split the second term of equation \eqref{eq:ap:full_process} for Q in 2:
\[
    \sum_{\nu = 1}^{n_b} \sum_{\nu' = 1}^{n_b} \sigma'(\lambda_{j,t}^\nu)\sigma'(\lambda_{l,t}^{\nu'})\mathcal{E}_t^\nu\mathcal{E}_t^{\nu'} \vec z^{\nu\top}_t \vec z_t^{\nu'} =
    \sum_{\nu = 1}^{n_b} \sigma'(\lambda_{j,t}^\nu) \sigma'(\lambda_{l,t}^\nu){\mathcal{E}_t^\nu}^2 \vec z^{\nu\top}_t \vec z^{\nu}
    +
    \sum_{\nu = 1}^{n_b} \sum_{\nu' = 1, \nu'\neq\nu}^{n_b} \sigma'(\lambda_{j,t}^\nu)\sigma'(\lambda_{l,t}^{\nu'})\mathcal{E}_t^\nu\mathcal{E}_t^{\nu'} \vec z^{\nu\top}_t \vec z_t^{\nu'}
\]
The first term is the standard gradient noise term that appears in Sadd\&Solla equations, while the second emerge from the large-batch, and has to be treated with new considerations.
Let's introduce now the component of the student vectors orthogonal to the teacher space 
\[
W_t^\bot \coloneqq W_t - M_tP^{-1}W^\star \text{ and consequently }
Q_t^\bot \coloneqq \left(W_t^\bot\right)^\top W_t^\bot = Q_t - M_tP^{-1}M_t^\top.
\]
We can also define the local fields of this subspace, with the interesting property of being independent with the teacher ones
\[
\vec \lambda^\bot \coloneqq W_t^\bot z \qquad
\vec \lambda^\bot \sim \mathcal N (0,Q_t^\bot) \quad
\text{Cov}[\vec \lambda^\star, \vec \lambda^\bot] = 0
\]

It is possible to choose a set \(\vec v_{\beta, t}\in\left(\Span{(W_t^\bot)}\cup\Span{(W^\star)}\right)^\bot\) of orthonormal vectors, such that 
\(
\left\{\vec w^\star_r,\vec w_{j,t}^\bot,\vec v_{\beta, t}\right\}_{r\in[k],j\in[p],\beta\in[d-p-k]}
\) is a basis of \(\mathbb R^d\).
Using the properties of the basis, we can write the identity matrix \(I_d\) as 
\[
    I_d = (W^\star)^\top P^{-1} W^\star + (W_t^\bot)^\top (Q_t^\bot)^{-1} W_t^\bot + \sum_{\beta=1}^{d-k-p} \vec v_{\beta, t} \vec v_{\beta, t}^\top
\]
We insert the identity matrix \(\vec z^\top \vec z\) with \(\vec z^\top I_d \vec z\). By recalling that $\vec \lambda^* = W^* \vec z, \vec \lambda^\bot = W_t^\bot \vec z$, we arrive to:
\begin{align}
    \sum_{\nu = 1}^{n_b} \sum_{\nu' = 1,\nu'\neq\nu}^{n_b} \sigma'(\lambda_{j,t}^\nu)\sigma'(\lambda_{l,t}^{\nu'})\mathcal{E}_t^\nu\mathcal{E}_t^{\nu'}
    \left(
        \left(\vec\lambda^{\nu\star}\right)^\top P^{-1} \vec\lambda^{\nu'\star} +
        \left(\vec\lambda_t^{\bot\nu}\right)^\top (Q_t^\bot)^{-1} \vec\lambda_t^{\bot\nu'} +
        \sum_{\beta=1}^{d-k-p} \langle \vec v_{\beta, t}, \vec z^\nu\rangle \langle\vec v_{\beta, t}, \vec z_t^{\nu'}\rangle
    \right)
\end{align}
Now, exploiting the relation:
\[
   \vec \lambda_t^\bot = \vec \lambda_t - M_tP^{-1} \vec \lambda^\star,
\]
and noting that the indeces \(\nu\) and \(\nu'\) are independent, all the sum we need to compute are just
\[
 \sum_{\nu = 1}^{n_b} \sigma'(\lambda_{j,t}^\nu)\lambda^{\star}_{r} \mathcal{E}_t^{\nu} \qquad 
 \sum_{\nu = 1}^{n_b} \sigma'(\lambda_{j,t}^\nu)\lambda_{l,t}^{\nu}
 \quad\text{and}\quad
 \sum_{\nu = 1}^{n_b} \sum_{\beta=1}^{d-k-p} \langle \vec v_{\beta, t}, \vec z_t^\nu\rangle 
\]

\paragraph{High dimensional limit} In our analysis we consider the high-dimensional limit \(d\to+\infty\) with the batch size going to infinite as well, with the scaling \(n_b = n_0 d^\mu\). Note that when \(\mu=0\) the intra-batch correlation disappear and we fall back to standard Saad\&Solla setting, given that the learning rate \(\gamma = \gamma_0 d^{-\delta}\) is small enough. Indeed, we can informally say that all the sums above
converge to their expected value
\begin{align}
    \frac{1}{n_b} \sum_{\nu = 1}^{n_b} \sigma'(\lambda_{j,t}^\nu)\lambda^{\star}_{r} \mathcal{E}_t^{\nu} \to& \mathbb{E}_{\Omega_t}\left[ \sigma'(\lambda_j)\lambda^{\star}_{r} \mathcal{E}  \right] = \psi_{jr,t} \label{eq:app:psi_def}
    \\
    \frac{1}{n_b} \sum_{\nu = 1}^{n_b} \sigma'(\lambda_{j,t}^\nu)\lambda_{l,t}^{\nu} \to& \mathbb{E}_{\Omega_t}\left[ \sigma'(\lambda_j)\lambda_{l} \mathcal{E} \right] = \phi^{\rm{GF}}_{jl,t} \label{eq:app:phigf_def}
    \\
    \frac{1}{n_b}\sum_{\nu = 1}^{n_b} \sigma'(\lambda_{j,t}^\nu) \sigma'(\lambda_{l,t}^\nu){\mathcal{E}_t^\nu}^2 \vec z^{\nu\top}_t \vec z^{\nu} \to& d\mathbb{E}_{\Omega_t}\left[ \sigma'(\lambda_j)\sigma'(\lambda_l)\mathcal{E}^2\right] = d\phi_{jl,t}^\text{GF} \label{eq:app:phihd_def}
    \\
    \sum_{\nu = 1}^{n_b} \sum_{\beta=1}^{d-k-p} \langle \vec v_{\beta, t}, \vec z_t^\nu\rangle \to& 0
\end{align}
Moreover, using \(\vec \lambda^\bot = \vec \lambda - MP^{-1} \vec \lambda^\star\) we have
\[
    \mathbb{E}_{\Omega_t}\left[\sigma'(\lambda_j)\mathcal E \vec\lambda^\bot\right] = \vec\phi_{j,t}^\text{GF} - M_tP^{-1}\vec \psi_{j,t}
\]

Plugging back in \eqref{eq:ap:full_process}, we finally obtain
\begin{align}
M_{jr,t+1} - M_{jr,t} \approx&  \frac{\gamma}{p}a_j \mathbb{E}_{\Omega_t}\left[ \sigma'(\lambda_j)\lambda^{\star}_{r} \mathcal{E} \right]\\
Q_{jl,t+1} - Q_{jl,t} \approx& \frac{\gamma}{p}\mathbb{E}_{\Omega_t}\left[\left(a_j\sigma'(\lambda_j)\lambda_{l}+a_l\sigma'(\lambda_{l,t}^\nu)\lambda_{j}\right)\mathcal{E}\right]  \\
    &+\frac{\gamma^2d}{p^2n_b}a_j a_l	\mathbb{E}_{\Omega_t}\left[ \sigma'(\lambda_j)\sigma'(\lambda_l)\mathcal{E}^2  \right] \\
    &+\mathbf{1}_{\{\mu\neq0\}}\frac{\gamma^2}{p^2} a_j a_l
        \left(\mathbb{E}_{\Omega_t}\left[\sigma'(\lambda_j)\mathcal E \left(\lambda^\star\right)^\top\right]P^{-1}\mathbb{E}_{\Omega_t}\left[\sigma'(\lambda_l)\mathcal E \lambda^\star\right] \right)
     \\
    &+\mathbf{1}_{\{\mu\neq0\}}\frac{\gamma^2}{p^2} a_j a_l\left(\mathbb{E}_{\Omega_t}\left[\sigma'(\lambda_j)\mathcal E (\lambda^\bot)^\top\right]\left(Q_t^\bot\right)^{-1}\mathbb{E}_{\Omega_t}\left[\sigma'(\lambda_l)\mathcal E \lambda^\bot\right]\right).
\end{align}
where the indicator function \(\mathbf{1}_{\{\mu\neq0\}}\) indicates that the last term is only present if the batch is large. If we want to make explicit all the dependencies in \(d\) (\(\gamma = \gamma_0 d^{-\delta}, n_b= n_0 d^\mu\)):
\begin{equation} \label{eq:app:process}\begin{split}
M_{jr,t+1} - M_{jr,t} \approx&  d^{-\delta}\frac{\gamma_0}{p}a_j \psi_{jr,t} = \Psi_{jr,t}\\
Q_{jl,t+1} - Q_{jl,t} \approx& d^{-\delta}\frac{\gamma_0}{p}\left(a_j\phi_{jl,t}^\text{GF} + a_l^\tau\phi_{lj,t}^\text{GF}\right)  \\
    &+d^{-2\delta+1-\mu}\frac{\gamma_0^2}{p^2n_0}a_j a_l \phi^\text{HD}_{jr} \\
    &+d^{-2\delta}\frac{\gamma_0^2}{p^2} a_j a_l \vec \phi^\text{GF}_{j} P^{-1} (\vec\phi^\text{GF}_{l})^\top\\
    &+d^{-2\delta}\frac{\gamma_0^2}{p^2} a_j a_l \left(\vec\phi_{j,t}^\text{GF} - M_tP^{-1}\vec \psi_{j,t}\right)(Q_t - M_tM_t^\top)^{-1}\left(\vec\phi^\text{GF}_l - M_tP^{-1}\vec \psi_{l}\right)^\top = \Phi_{jl,t}
\end{split}\end{equation}
These equations are the starting point for all our considerations, both when investigating weakly correlation and when characterizing the exact dynamics.

Indeed, when we have a cold start we have to take into account that \(\psi_{jr,t}, \phi_{jl,t}^\text{GF}, \phi^\text{HD}_{jl}\) can also go to 0 when \(d\to+\infty\). A careful analysis for the leading terms of these equations around initializations will also give us infromation on the behaviour of the system, and ultimately it will allow to have rules on how to scale \(\delta\) and \(\mu\) to have the best performance. An example for these analysis for \emph{generalized linear model} is provided in Appendix \ref{sec:app:derivation_of_lb_equations}.

On the other hand, when can also assume \(\psi_{jr,t}, \phi_{jl,t}^\text{GF}, \phi^\text{HD}_{jl} = O_d(1)\) for all the dynamics\footnote{This happens if \(\ell\le1\) or when we provide an informed initialization.} and use the equations for an asymptotic description. Clearly, depending on the values of \(\delta\) and \(\mu\), not all terms are present in the limiting equations. A detailed discussion about this is provided in Section~\ref{sec:main:warm_start}.

\paragraph{Spherical projection} When we consider the spherical gradient descent, i.e., the modification of eq.~\eqref{eq:gd_update_weight} 
\begin{align}
    \vec w_{j,t + 1} = \frac{
    \vec w_{j,t } - \gamma  \nabla_{\vec w_j} \ell_t}
    {|| \vec w_{j,t } - \gamma  \nabla_{\vec w_j}\ell_t||},
    \label{eq:spherical_gd_update_weight}
\end{align}
the overlap dynamics for the spherical large batch SGD can be then approximated as 
\begin{equation}
\begin{split}
    M_{jr,t+1} - M_{jr,t} \approx& \Psi_{jr,t}{(\Omega)} - \frac{M_{jr,t}}{2}\Phi_{jj,t}{(\Omega)} \\
    Q_{jl,t+1} - Q_{jl,t} \approx& \Phi_{jl,t}{(\Omega)} - \frac{1}{2}Q_{jl,t}\left(\Phi_{jj,t}{(\Omega)}+\Phi_{ll,t}{(\Omega)}\right) \\
\end{split}
\end{equation}
This derivation follows from a Taylor expansion of the denominator needed to project the update equations on the sphere. As final note, this aproximation only holds when \(\gamma\) is vanishing when \(d\to+\infty\): that's why we need \(\gamma = o_d(1)\) in Propposition~\ref{prop:main:exact_asympt}. When \(\gamma\) is too large, all the other orders of Taylor expansion play a role, and we can't have a simple expression for the exact evoluton, even near initialization. Neverthless, the first order expansion is a lower bound of the true dynamic that can provided guarantee of learning in some cases:
\begin{equation}\label{eq:app:generallowerbound}
\begin{split}
    M_{jr,t+1} - M_{jr,t} \ge& \Psi_{jr,t}{(\Omega)} - \frac{M_{jr,t}}{2}\Phi_{jj,t}{(\Omega)} \\
    Q_{jl,t+1} - Q_{jl,t} \ge& \Phi_{jl,t}{(\Omega)} - \frac{1}{2}Q_{jl,t}\left(\Phi_{jj,t}{(\Omega)}+\Phi_{ll,t}{(\Omega)}\right) \\
\end{split}
\end{equation}

\section{Special case: committee machine with matching architecture} \label{sec:app:matching_architectures}
We consider a separable teacher, more precisely it is a committee machine with $k$ hidden units, i.e., $$f_*(\vec z) = \frac{1}{k} \sum_{r=1}^k a^*_r \sigma(\lambda^\star_r)$$ where we additionally consider a matched architecture in which the student and teacher share the same activation function $\sigma$.

As we discussed in Section~\ref{sec:app:derivation_of_lb_equations}, the activation and the target appear just in the expected values of Equations~\eqref{eq:app:psi_def},\eqref{eq:app:phigf_def}~and~\eqref{eq:app:phihd_def}, that can be further simplified for matching architectures
\begin{align}
    \psi_{jr,t} = \mathbb{E}_{\Omega_t}\left[ \sigma'(\lambda_j)\lambda^{\star}_{r} \mathcal{E} \right] &= \frac{1}{k}\sum_{t=1}^k a^*_t \mathbb{E}_{\Omega_t}\left[
\sigma^\prime(\lambda_j)\lambda^*_r \sigma(\lambda^*_t)
    \right]
    - \frac{1}{p}\sum_{s=1}^p a_s \mathbb{E}_{\Omega_t}\left[
\sigma^\prime(\lambda_j)\lambda^*_r \sigma(\lambda_s)
    \right] \\
\phi^{\rm{GF}}_{jl,t} = \mathbb{E}_{\Omega_t}\left[ \sigma'(\lambda_j)\lambda_{l} \mathcal{E} \right] &=  \frac{1}{k}\sum_{t=1}^k a^*_t \mathbb{E}_{\Omega_t}\left[
\sigma^\prime(\lambda_j)\lambda_l \sigma(\lambda^*_t)
    \right]
    - \frac{1}{p}\sum_{s=1}^p a_s \mathbb{E}_{\Omega_t}\left[
\sigma^\prime(\lambda_j)\lambda_l \sigma(\lambda_s)
    \right] \\ 
\phi^{\rm{HD}}_{jl,t} = \mathbb{E}_{\Omega_t}\left[ \sigma'(\lambda_j)\sigma'(\lambda_l)\mathcal{E}^2  \right] &= \frac{1}{k^2} \sum_{r, t=1}^{k} a^*_r a^*_{t} \mathbb{E}_{\Omega_t}\left[\sigma^{\prime}\left(\lambda_j\right) \sigma^{\prime}\left(\lambda_l\right) \sigma\left(\lambda_r^*\right) \sigma\left(\lambda_{t}^*\right)\right] \\
& +\frac{1}{p^2} \sum_{s,u=1}^p a_s a_{u}\mathbb{E}_{\Omega_t}\left[\sigma^{\prime}\left(\lambda_j\right) \sigma^{\prime}\left(\lambda_l\right) \sigma\left(\lambda_s\right) \sigma\left(\lambda_{u}\right)\right] \\
& -\frac{2}{pk} \sum_{s=1}^p \sum_{r=1}^k a^*_r a_s \mathbb{E}_{\Omega_t}\left[\sigma^{\prime}\left(\lambda_j\right) \sigma^{\prime}\left(\lambda_l\right) \sigma\left(\lambda_r^*\right) \sigma\left(\lambda_{s}\right)\right] \\
& +\Delta  \mathbb{E}_{\Omega_t}\left[\sigma^{\prime}\left(\lambda_j\right) \sigma^{\prime}\left(\lambda_l\right)\right]
\end{align}
In addition, we can also express the population risk as
\begin{equation}
\mathcal{R}_t= \frac 12 \mathbb E_{\Omega_t} \left[\mathcal{E}^2\right] = \frac\Delta2 + \frac{1}{p^2}\sum_{s,u}^p a_s a_u \mathbb E_{\Omega_t} \left[\sigma(\lambda_s)\sigma(\lambda_u)\right] +
\frac{1}{k^2}\sum_{r,t}^k a^\star_r a^\star_t \mathbb E_{\Omega_t} \left[\sigma(\lambda^\star_r)\sigma(\lambda^\star_t)\right]
-\frac{2}{pk} \sum_{s,r=1}^{p,k} a_s a^\star_r\mathbb E_{\Omega_t} \left[\sigma(\lambda_s)\sigma(\lambda^\star_r)\right].
\end{equation}
We introduce auxiliary functions to simplify the mathematical notations: 
\begin{align}
I_2(\omega_{\alpha\alpha},\omega_{\alpha\beta},\omega_{\beta\beta}) &=   \mathbb E\left[ \sigma(\lambda_\alpha)\sigma(\lambda_\beta)\right] \\
I_3(\omega_{\alpha\alpha},\omega_{\alpha\beta}, \omega_{\alpha\gamma},\omega_{\beta\beta}, \omega_{\beta\gamma},\omega_{\gamma\gamma}) &= \mathbb E\left[ \sigma^\prime(\lambda_\alpha) \lambda_\beta\sigma(\gamma)\right] 
\\
 I_4(
      \omega_{\alpha\alpha},
      \omega_{\alpha\beta},
      \omega_{\alpha\gamma},
      \omega_{\alpha\delta},
      \omega_{\beta\beta},
      \omega_{\beta\gamma},
      \omega_{\beta\delta},
      \omega_{\gamma\gamma},
      \omega_{\gamma\delta},
      \omega_{\delta\delta}
    ) &= \mathbb E\left[\sigma^\prime(\lambda_\alpha) \sigma^\prime(\lambda_\beta)\sigma(\lambda_\alpha)\sigma(\lambda_\beta)\right]  \\
I^\text{noise}_2(\omega_{\alpha\alpha},\omega_{\alpha\beta},\omega_{\beta\beta}) &= \mathbb E\left[\sigma^\prime(\lambda_\alpha) \sigma^\prime(\lambda_\beta)\right] .
\label{eq:define_Is}
\end{align}
\noindent where we introduced the correlation $\omega_{\alpha\beta} = \mathbb E [\lambda_\alpha \lambda_\beta]$, where $(\alpha , \beta)$ are indices running on either the teacher or the student components. Dropping the time index for clarity, we finally obtain:
\begin{align}
    \psi_{jr} =& \frac{1}{k}\sum_{t=1}^k a^*_t I_3(Q_{jj}, M_{jr}, M_{jt}, P_{rr}, P_{rt}, P_{tt}) - \frac{1}{p}\sum_{s=1}^p a_sI_3(Q_{jj}, M_{jr}, Q_{js}, P_{rr}, M_{sr}, Q_{ss}) \\
    \phi^{\rm{GF}}_{jl} =&  \frac{1}{k}\sum_{t=1}^k a^*_t I_3(Q_{jj}, Q_{jl}, M_{jt}, Q_{ll}, M_{lt}, P_{tt}) - \frac{1}{p}\sum_{s=1}^p a_s I_3(Q_{jj}, Q_{jl}, Q_{js}, Q_{ll}, Q_{ls}, Q_{ss}) \\
    \phi^{\rm{HD}}_{jr} =&
    \frac{1}{k^2} \sum_{r, t=1}^{k} a^*_r a^*_{t} I_4(Q_{jj}, Q_{jl}, M_{jr}, M_{jt}, Q_{ll}, M_{lr}, M_{lt}, P_{rr}, P_{rt}, P_{tt})\\
    & +\frac{1}{p^2} \sum_{s,u=1}^p a_s a_{u} I_4(Q_{jj}, Q_{jl}, Q_{js}, Q_{ju}, Q_{ll}, Q_{ls}, Q_{lu}, Q_{ss}, Q_{su}, Q_{uu})\\
    & -\frac{2}{pk} \sum_{s=1}^p \sum_{r=1}^k a^*_r a_s  I_4(Q_{jj}, Q_{jl}, Q_{js}, M_{jr}, Q_{ll}, Q_{ls}, M_{lr}, Q_{ss}, M_{sr}, P_{rr})\\
    & +\Delta  I^\text{noise}_2(Q_{jj},Q_{jl},Q_{ll})\\
    \mathcal{R} =& \frac\Delta2 + \frac{1}{p^2}\sum_{s,u}^p a_s a_u  I_2(Q_{ss},Q_{su},Q_{uu}) +
\frac{1}{k^2}\sum_{r,t}^k a^\star_r a^\star_t I_2(P_{rr},P_{rt},P_{tt})
\\
&-\frac{2}{pk} \sum_{s,r=1}^{p,k} a_s a^\star_r I_2(Q_{ss},M_{sr},P_{rr})
\end{align}
When analizing a matching architecture setting, we just need to specify \(I_2,I_3,I_4\) and \(I^\text{noise}_2\). In the following sections we provide the explicit expersion for all the case used in numerical simulation inside this paper.

\subsection{Analytic case $\sigma =\erf\left(\sfrac{\cdot}{\sqrt{2}}\right)$}
The expressions can be found in the appendix of \cite{veiga2022phase}.

\subsection{Analytic case $\sigma =\text{He}_2$}
\label{sec:app:H2_theory}
We report here the auxiliary functions: 
\begin{align}
&I_2(\omega_{\alpha\alpha},\omega_{\alpha\beta},\omega_{\beta\beta}) = \mathbb E\left[(\lambda_\alpha^2-1)(\lambda_\beta^2-1)\right] = \omega_{\alpha\alpha} \omega_{\beta\beta} + 2 \omega_{\alpha\beta}^2 - \omega_{\alpha\alpha} - \omega_{\beta\beta} + 1 \\
&I_3(\omega_{\alpha\alpha},\omega_{\alpha\beta}, \omega_{\alpha\gamma},\omega_{\beta\beta}, \omega_{\beta\gamma},\omega_{\gamma\gamma}) = 2\mathbb E \left[\lambda_\alpha\lambda_\beta(\lambda^2_\gamma-1)\right] = 2 \omega_{\alpha \beta} \omega_{\gamma \gamma}+4 \omega_{\alpha \gamma} \omega_{\beta \gamma} - 2\omega_{\alpha\beta} \\
 &I_4(
      \omega_{\alpha\alpha},
      \omega_{\alpha\beta},
      \omega_{\alpha\gamma},
      \omega_{\alpha\delta},
      \omega_{\beta\beta},
      \omega_{\beta\gamma},
      \omega_{\beta\delta},
      \omega_{\gamma\gamma},
      \omega_{\gamma\delta},
      \omega_{\delta\delta}
    ) = 4\mathbb E\left[\lambda_\alpha\lambda_\beta(\lambda^2_\gamma-1)(\lambda^2_\delta-1)\right]
    \\
&I^\text{noise}_2(\omega_{\alpha\beta}) = 4 \mathbb E \left[\lambda_\alpha\lambda_\beta\right] = 4 \omega_{\alpha \beta}
\end{align} 
We now work on the different terms: 
\begin{align} 
4 \mathbb{E}\left[\lambda^\alpha \lambda^\beta\left(\lambda^\gamma\right)^2\left(\lambda^\delta\right)^2\right]=& 4 \omega_{\alpha \beta} \omega_{\gamma \gamma} \omega_{\delta \delta}+8 \omega_{\alpha \beta} \omega_{\gamma \delta}^2+8 \omega_{\alpha \gamma} \omega_{\beta \gamma} \omega_{\delta \delta^{+}} \\
& 16 \omega_{\alpha \gamma} \omega_{\beta \delta} \omega_{\gamma \delta}+16 \omega_{\alpha \delta} \omega_{\beta \gamma} \omega_{\gamma \delta}+8 \omega_{\alpha \delta} \omega_{\beta \delta} \omega_{\gamma \gamma}  \\
4\mathbb{E} [\lambda_\alpha \lambda_\beta \lambda_\gamma^2] &= 4\omega_{\alpha \beta} \omega_{\gamma \gamma} + 8 \omega_{\alpha \gamma} \omega_{\beta \gamma}\\
4\mathbb{E}[\lambda_\alpha \lambda_\beta \lambda_\delta^2] &= 4\omega_{\alpha \beta} \omega_{\delta \delta} + 8 \omega_{\alpha \delta} \omega_{\beta \delta} \\
4 \mathbb{E}[\lambda_\alpha \lambda_\beta] &= 4\omega_{\alpha \beta}
\end{align}

And then we arrive to:
\begin{align}
    I_4 &= 4 \omega_{\alpha \beta} \omega_{\gamma \gamma} \omega_{\delta \delta}+8 \omega_{\alpha \beta} \omega_{\gamma \delta}^2+8 \omega_{\alpha \gamma} \omega_{\beta \gamma} \omega_{\delta \delta} \\
& 16 \omega_{\alpha \gamma} \omega_{\beta \delta} \omega_{\gamma \delta}+16 \omega_{\alpha \delta} \omega_{\beta \gamma} \omega_{\gamma \delta}+8 \omega_{\alpha \delta} \omega_{\beta \delta} \omega_{\gamma \gamma} \\
& -4\omega_{\alpha \beta} \omega_{\gamma \gamma} - 8 \omega_{\alpha \gamma} \omega_{\beta \gamma} -  4\omega_{\alpha \beta} \omega_{\delta \delta} - 8 \omega_{\alpha \delta} \omega_{\beta \delta}  + 4 \omega_{\alpha \beta}
\end{align}

\subsection{Analytic case $\sigma = \text{He}_3$}
\label{sec:app:theory_H3}
We report the auxiliary function for this case below. 
\[\begin{split}
I^\text{noise}_2(\omega_{\alpha\alpha},\omega_{\alpha\beta},\omega_{\beta\beta}) \coloneqq& \mathbb E \left[(3\lambda_\alpha^2-3)(3\lambda_\beta^2-3)\right] \\
=& 9 - 9\omega_{\alpha\alpha} + 18\omega_{\alpha\beta}^2 - 9\omega_{\beta\beta} + 9\omega_{\alpha\alpha}\omega_{\beta\beta}
\end{split}\]
\[\begin{split}
I_2(\omega_{\alpha\alpha},\omega_{\alpha\beta},\omega_{\beta\beta}) \coloneqq& \mathbb E \left[(\lambda_\alpha^3-3\lambda_\alpha)(\lambda_\beta^3-3\lambda_\beta)\right] \\
=& 9\omega_{\alpha\beta} - 9\omega_{\alpha\alpha}\omega_{\alpha\beta} + 6\omega_{\alpha\beta}^3 - 9\omega_{\alpha\beta}\omega_{\beta\beta} + 9\omega_{\alpha\alpha}\omega_{\alpha\beta}\omega_{\beta\beta}
\end{split}\]
\begin{equation}\begin{split}
I_3(\omega_{\alpha\alpha},\omega_{\alpha\beta}, \omega_{\alpha\gamma},\omega_{\beta\beta}, \omega_{\beta\gamma},\omega_{\gamma\gamma}) \coloneqq& \mathbb E \left[(3\lambda_\alpha^2-3)\lambda_\beta(\lambda^3_\gamma-3\lambda_\gamma)\right] \\
=& -18\omega_{\alpha\beta}\omega_{\alpha\gamma} + 9\omega_{\beta\gamma} - 9\omega_{\alpha\alpha}\omega_{\beta\gamma} + 18\omega_{\alpha\gamma}^2\omega_{\beta\gamma} + \\
& 18\omega_{\alpha\beta}\omega_{\alpha\gamma}\omega_{\gamma\gamma} - 9\omega_{\beta\gamma}\omega_{\gamma\gamma} + 9\omega_{\alpha\alpha}\omega_{\beta\gamma}\omega_{\gamma\gamma}
\end{split}\end{equation}
\begin{equation}\begin{split}
    I_4(\cdots) \coloneqq& \mathbb E\left[(3\lambda_\alpha^2-3)(3\lambda_\beta^2-3)(\lambda_\gamma^3-3\lambda_\gamma)(\lambda_\delta^3-3\lambda_\delta)\right] \\
    =& -162\omega_{\alpha\gamma}\omega_{\alpha\delta} + 162\omega_{\alpha\gamma}\omega_{\alpha\delta}\omega_{\beta\beta} + 324\omega_{\alpha\beta}\omega_{\alpha\delta}\omega_{\beta\gamma} - 324\omega_{\alpha\gamma}\omega_{\alpha\delta}\omega_{\beta\gamma}^2 + 324\omega_{\alpha\beta}\omega_{\alpha\gamma}\omega_{\beta\delta} - 162\omega_{\beta\gamma}\omega_{\beta\delta} + \\
    & 162\omega_{\alpha\alpha}\omega_{\beta\gamma}\omega_{\beta\delta} - 324\omega_{\alpha\gamma}^2\omega_{\beta\gamma}\omega_{\beta\delta} - 324\omega_{\alpha\delta}^2\omega_{\beta\gamma}\omega_{\beta\delta} - 324\omega_{\alpha\gamma}\omega_{\alpha\delta}\omega_{\beta\delta}^2 + 162\omega_{\alpha\gamma}\omega_{\alpha\delta}\omega_{\gamma\gamma} - 162\omega_{\alpha\gamma}\omega_{\alpha\delta}\omega_{\beta\beta}\omega_{\gamma\gamma} - \\
    & 324\omega_{\alpha\beta}\omega_{\alpha\delta}\omega_{\beta\gamma}\omega_{\gamma\gamma} - 324\omega_{\alpha\beta}\omega_{\alpha\gamma}\omega_{\beta\delta}\omega_{\gamma\gamma} + 162\omega_{\beta\gamma}\omega_{\beta\delta}\omega_{\gamma\gamma} - 162\omega_{\alpha\alpha}\omega_{\beta\gamma}\omega_{\beta\delta}\omega_{\gamma\gamma} + 324\omega_{\alpha\delta}^2\omega_{\beta\gamma}\omega_{\beta\delta}\omega_{\gamma\gamma} + \\
    & 324\omega_{\alpha\gamma}\omega_{\alpha\delta}\omega_{\beta\delta}^2\omega_{\gamma\gamma} + 81\omega_{\gamma\delta} - 81\omega_{\alpha\alpha}\omega_{\gamma\delta} + 162\omega_{\alpha\beta}^2\omega_{\gamma\delta} + 162\omega_{\alpha\gamma}^2\omega_{\gamma\delta} + 162\omega_{\alpha\delta}^2\omega_{\gamma\delta} - 81\omega_{\beta\beta}\omega_{\gamma\delta} + 81\omega_{\alpha\alpha}\omega_{\beta\beta}\omega_{\gamma\delta} - \\
    & 162\omega_{\alpha\gamma}^2\omega_{\beta\beta}\omega_{\gamma\delta} - 162\omega_{\alpha\delta}^2\omega_{\beta\beta}\omega_{\gamma\delta} - 648\omega_{\alpha\beta}\omega_{\alpha\gamma}\omega_{\beta\gamma}\omega_{\gamma\delta} + 162\omega_{\beta\gamma}^2\omega_{\gamma\delta} - 162\omega_{\alpha\alpha}\omega_{\beta\gamma}^2\omega_{\gamma\delta} + 324\omega_{\alpha\delta}^2\omega_{\beta\gamma}^2\omega_{\gamma\delta} - \\
    & 648\omega_{\alpha\beta}\omega_{\alpha\delta}\omega_{\beta\delta}\omega_{\gamma\delta} + 1296\omega_{\alpha\gamma}\omega_{\alpha\delta}\omega_{\beta\gamma}\omega_{\beta\delta}\omega_{\gamma\delta} + 162\omega_{\beta\delta}^2\omega_{\gamma\delta} - 162\omega_{\alpha\alpha}\omega_{\beta\delta}^2\omega_{\gamma\delta} + 324\omega_{\alpha\gamma}^2\omega_{\beta\delta}^2\omega_{\gamma\delta} - \\
    & 81\omega_{\gamma\gamma}\omega_{\gamma\delta} + 81\omega_{\alpha\alpha}\omega_{\gamma\gamma}\omega_{\gamma\delta} - 162\omega_{\alpha\beta}^2\omega_{\gamma\gamma}\omega_{\gamma\delta} - 162\omega_{\alpha\delta}^2\omega_{\gamma\gamma}\omega_{\gamma\delta} + 81\omega_{\beta\beta}\omega_{\gamma\gamma}\omega_{\gamma\delta} - 81\omega_{\alpha\alpha}\omega_{\beta\beta}\omega_{\gamma\gamma}\omega_{\gamma\delta} + \\
    & 162\omega_{\alpha\delta}^2\omega_{\beta\beta}\omega_{\gamma\gamma}\omega_{\gamma\delta} + 648\omega_{\alpha\beta}\omega_{\alpha\delta}\omega_{\beta\delta}\omega_{\gamma\gamma}\omega_{\gamma\delta} - 162\omega_{\beta\delta}^2\omega_{\gamma\gamma}\omega_{\gamma\delta} + 162\omega_{\alpha\alpha}\omega_{\beta\delta}^2\omega_{\gamma\gamma}\omega_{\gamma\delta} - \\
    & 324\omega_{\alpha\gamma}\omega_{\alpha\delta}\omega_{\gamma\delta}^2 + 324\omega_{\alpha\gamma}\omega_{\alpha\delta}\omega_{\beta\beta}\omega_{\gamma\delta}^2 + 648\omega_{\alpha\beta}\omega_{\alpha\delta}\omega_{\beta\gamma}\omega_{\gamma\delta}^2 + 648\omega_{\alpha\beta}\omega_{\alpha\gamma}\omega_{\beta\delta}\omega_{\gamma\delta}^2 - 324\omega_{\beta\gamma}\omega_{\beta\delta}\omega_{\gamma\delta}^2 + \\
    & 324\omega_{\alpha\alpha}\omega_{\beta\gamma}\omega_{\beta\delta}\omega_{\gamma\delta}^2 + 54\omega_{\gamma\delta}^3 - 54\omega_{\alpha\alpha}\omega_{\gamma\delta}^3 + 108\omega_{\alpha\beta}^2\omega_{\gamma\delta}^3 - 54\omega_{\beta\beta}\omega_{\gamma\delta}^3 + 54\omega_{\alpha\alpha}\omega_{\beta\beta}\omega_{\gamma\delta}^3 + \\
    & 162\omega_{\alpha\gamma}\omega_{\alpha\delta}\omega_{\delta\delta} - 162\omega_{\alpha\gamma}\omega_{\alpha\delta}\omega_{\beta\beta}\omega_{\delta\delta} - 324\omega_{\alpha\beta}\omega_{\alpha\delta}\omega_{\beta\gamma}\omega_{\delta\delta} + 324\omega_{\alpha\gamma}\omega_{\alpha\delta}\omega_{\beta\gamma}^2\omega_{\delta\delta} - \\
    & 324\omega_{\alpha\beta}\omega_{\alpha\gamma}\omega_{\beta\delta}\omega_{\delta\delta} + 162\omega_{\beta\gamma}\omega_{\beta\delta}\omega_{\delta\delta} - 162\omega_{\alpha\alpha}\omega_{\beta\gamma}\omega_{\beta\delta}\omega_{\delta\delta} + 324\omega_{\alpha\gamma}^2\omega_{\beta\gamma}\omega_{\beta\delta}\omega_{\delta\delta} - \\
    & 162\omega_{\alpha\gamma}\omega_{\alpha\delta}\omega_{\gamma\gamma}\omega_{\delta\delta} + 162\omega_{\alpha\gamma}\omega_{\alpha\delta}\omega_{\beta\beta}\omega_{\gamma\gamma}\omega_{\delta\delta} + 324\omega_{\alpha\beta}\omega_{\alpha\delta}\omega_{\beta\gamma}\omega_{\gamma\gamma}\omega_{\delta\delta} + \\
    & 324\omega_{\alpha\beta}\omega_{\alpha\gamma}\omega_{\beta\delta}\omega_{\gamma\gamma}\omega_{\delta\delta} - 162\omega_{\beta\gamma}\omega_{\beta\delta}\omega_{\gamma\gamma}\omega_{\delta\delta} + 162\omega_{\alpha\alpha}\omega_{\beta\gamma}\omega_{\beta\delta}\omega_{\gamma\gamma}\omega_{\delta\delta} - 81\omega_{\gamma\delta}\omega_{\delta\delta} + \\
    & 81\omega_{\alpha\alpha}\omega_{\gamma\delta}\omega_{\delta\delta} - 162\omega_{\alpha\beta}^2\omega_{\gamma\delta}\omega_{\delta\delta} - 162\omega_{\alpha\gamma}^2\omega_{\gamma\delta}\omega_{\delta\delta} + 81\omega_{\beta\beta}\omega_{\gamma\delta}\omega_{\delta\delta} - 81\omega_{\alpha\alpha}\omega_{\beta\beta}\omega_{\gamma\delta}\omega_{\delta\delta} + \\
    & 162\omega_{\alpha\gamma}^2\omega_{\beta\beta}\omega_{\gamma\delta}\omega_{\delta\delta} + 648\omega_{\alpha\beta}\omega_{\alpha\gamma}\omega_{\beta\gamma}\omega_{\gamma\delta}\omega_{\delta\delta} - 162\omega_{\beta\gamma}^2\omega_{\gamma\delta}\omega_{\delta\delta} + 162\omega_{\alpha\alpha}\omega_{\beta\gamma}^2\omega_{\gamma\delta}\omega_{\delta\delta} + \\
    & 81\omega_{\gamma\gamma}\omega_{\gamma\delta}\omega_{\delta\delta} - 81\omega_{\alpha\alpha}\omega_{\gamma\gamma}\omega_{\gamma\delta}\omega_{\delta\delta} + 162\omega_{\alpha\beta}^2\omega_{\gamma\gamma}\omega_{\gamma\delta}\omega_{\delta\delta} - 81\omega_{\beta\beta}\omega_{\gamma\gamma}\omega_{\gamma\delta}\omega_{\delta\delta} + \\
    & 81\omega_{\alpha\alpha}\omega_{\beta\beta}\omega_{\gamma\gamma}\omega_{\gamma\delta}\omega_{\delta\delta}
\end{split}
\end{equation}
\section{Weak recovery with Generalized Linear Models} \label{app:sec:weakrecoveryglm}
In this section, we restrict our analysis to matching architectures with \(p=k=1\), i.e. \emph{Generalized Linear Models} (GLMs). Moreover, we consider as activation function the Hermite polynomials \(\sigma = \text{He}_\ell\), so that we can have control on the information exponent of the problem.
Finally, the training algorithm is \emph{projected SGD}, given by Equation~\eqref{eq:main:projected_sgd}.
We will also assume that \(a=a_\star=1\) throughout all the dynamics.

Let us start by noticing that the set of sufficient statistics reduces to just one single parameter \(m=\langle\vec w,\vec w_\star\rangle\).
Retracing backward all the steps of the Sections \ref{sec:app:matching_architectures}~and~\ref{sec:app:derivation_of_lb_equations} up to \eqref{eq:app:generallowerbound}, we can obtain the lower bound for the update of \(m\). As examples the explicit equation for \(\sigma = \text{He}_2\) is
\[\begin{split}
    m_{t+1}-m_t \ge& \gamma_0 d^{-\delta}\Bigg[4m_t - 4m_t^3 - d^{-\delta}\gamma_0\mathbf{1}_{\{\mu\neq0\}}\left(8m_t-8m_t^3\right) +\\ 
        &+d^{-\delta+1-\mu}\frac{\gamma_0}{n_0}\left(24m_t-24m_t^3+2m_t^2\Delta\right)\Bigg],
\end{split}\]
while for \(\sigma = \text{He}_3\) is
\[\begin{split}
    m_{t+1}-m_t \ge& \gamma_0 d^{-\delta}\Bigg[18m_t^2 - 18m_t^4 - d^{-\delta}\gamma_0\mathbf{1}_{\{\mu\neq0\}}\left(162m_t + 324m_t^4 - 162m_t^5\right) +\\ 
&+d^{-\delta+1-\mu}\frac{\gamma_0}{n_0}\left(-1728m_t - 648m_t^3 + 3348m_t^4-972m_t^5-9\Delta m_t^3\right)\Bigg]. 
\end{split}\]
In general, for an Hermite polynomial activation \(\text{He}_\ell\), the equation of the evolution of \(m\) around \(m=0\) is given by
\begin{equation} \label{eq:app:m_expansion}
    m_{t+1}-m_t \ge d^{-\delta}\beta_\ell m^{\ell-1} 
        -d^{-\delta} \left(d^{-\delta+1-\mu}\alpha_\ell + d^{-\delta}\mathbf{1}_{\{\mu\neq0\}}\phi_\ell   \right)m
\end{equation}
where we fixed \(\gamma_0=n_0=1\) for simplicity; \(\alpha_\ell, \beta_\ell, \phi_\ell\) are constants. For computing the full equations for any generic \(\ell\), with the constants values, we refer to the Mathematica notebook published in the repository of this work.

At initialization, \(m_0 = \sfrac{1}{\sqrt{d}}\). The crucial observation is that a sufficient condition to escape initialization is to have equation~\eqref{eq:app:m_expansion} being  \emph{expansive}, namely \(\Delta m>0\) for \(m\) close to zero. This can be true if and only if \(\left(d^{-\delta+1-\mu}\alpha_\ell + d^{-\delta}\mathbf{1}_{\{\mu\neq0\}}\phi_\ell   \right)m\) is negligible when compared to \(\beta_\ell m^{\ell-1}\), so that the equation for \(m\) is lower-bounded by
\begin{equation} \label{eq:app:m_expansive}
     \frac{\Delta m}{\Delta t} \ge \beta_\ell m^{\ell-1} + \text{h.o.t} \quad\text{with}\quad \Delta t = d^{-\delta}
\end{equation}
Assuming that \(\gamma = o_d(1)\), the bound becomes tight and we can also derive some sharp characterization of the escaping time. By simple arguments on differential equations, we can claim that the order of magnitude of steps needed to escape the initial mediocrity is given by
\[
 T = \begin{cases}
     O_d\left(\frac{1}{\Delta t}\right) & \ell=1 \\
     O_d\left(\frac{\log m_0}{\Delta t}\right) & \ell=2 \\
     O_d\left(\frac{1} {m_0^{\ell-2}\Delta t}\right) & \ell\ge3 \\
 \end{cases},
\]
remembering that \(m_0=\sfrac{1}{\sqrt{d}}\) and \(\Delta t = d^{-\delta}\), these lead to
\begin{equation} \label{eq:app:escapingtimes}
\log_d T \sim \begin{cases}
    \max(\delta,0) & \ell=1 \\
    \log\log d + \delta & \ell=2 \\
    \delta - 1 +\sfrac{\ell}{2} & \ell\ge3 \\
 \end{cases}.
\end{equation}

It's clear that to escape as fast as possible, we want \(\delta\) to be the smallest possible, or in other words, having the learning rate as large as possible. Obviously, \(\delta\) is constrained by the values that make equation~\eqref{eq:app:m_expansive} true (or equivalently by the assumptions of the formal proof in Appendix~\ref{sec:app:proofs}).
The phase diagram of the allowed value of \(\delta\) and \(\mu\) is summarized in Figure~\ref{fig:app:deltamu_phasediagram}: the green region is where the equations for \(m\) is expansive, the red and the yellow region is where the equations in attractive, so there is no escaping, the purple region is where we can't do expansion because the learning rate is too large and the process diverge. Figure~\ref{fig:cold_start_pd} in the main text shows the same result in terms of \(T\) and \(n_b\), when \(\ell\ge3\).
\begin{figure}
    \centering
    \includegraphics[width=0.9\textwidth]{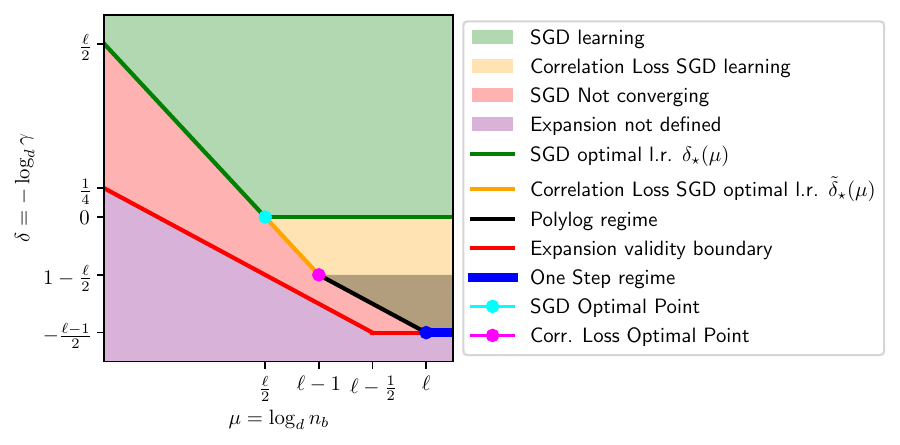}
    \caption{\textbf{Phase diagram for the learning rate:} The plot identifies different learning behaviors of standard SGD and \emph{Correlation Loss SGD} for different values of learning rate and batch size when considering randomly initialized networks, i.e. $m_0 = O(\sfrac{1}{\sqrt{d}})$.}
    \label{fig:app:deltamu_phasediagram}
\end{figure}

\subsection{Correlations loss SGD}
When using \emph{Correlation Loss SGD}, Equation~\eqref{eq:app:m_expansion} rewrite as 
\[
    \frac{\Delta m}{\Delta t} \ge \beta_\ell m^{\ell-1} 
        -d^{-\delta+1-\mu}\alpha_\ell m
\]
effectively removing a constraint on the possible values of \(\delta\). The modified version of SGD can use smaller values of \(\delta\) for escaping the initial condition, reaching regions in the phase diagram that are not allowed for SGD: this is colored in yellow in Figures~\ref{fig:app:deltamu_phasediagram}~and~\ref{fig:cold_start_pd}. Of course, if the learning rate becomes too large, all the theory does not work anymore (purple region in the diagram). In Figure~\ref{fig:cold_start_pd} we show that the number of steps needed to weakly recover can be pushed down to be smaller than any power scaling with \(d\) (black and blue line on the x-axis). The picture becomes clearer if we look at the same diagram in terms of \((\mu,\delta)\): Correlation Loss SGD can be used with learning very large learning rates (\(1-\sfrac\ell2>\delta>-\sfrac{(\ell-1)}{2}\)) such that the escaping times is \(T=O(\mathrm{polylog}(d))\), as proved in Theorem~\ref{thm:main:no_yhat_weak_recovery}.
We believe that the true number of steps is actually \(T=O(\mathrm{log}(d))\), but we could not find any formal proof; in Section~\ref{app:sec:polylog}, we were able to show that for \(\ell=2\) we have \(T=O(\mathrm{log}(d))\), relying the result on numerical integration of our asymptotic theory. Lastly, if the learning rate is of order \(\gamma=O(d^{-\delta})=O(d^{\sfrac{\ell-1}{2}})\) we recover the result of \cite{{dandi2023twolayer}}: the target can be weakly recovered in just one step, when the batch size is \(n_b>O(d^\ell)\).

\subsection{Simple example: retrieving \cite{arous2021online}}
In this section, we want to show how to find the same result presented in \cite{arous2021online} starting from our formalism.
There, online one-pass SGD is considered, meaning \(n_b=1 \implies \mu = 0\) in our context. Moreover, a vanishing learning rate is assumed, which implies \(\delta>0\), and all the bounds for the evolutions of \(m\) are tight. The condition for expansiveness of equation~\eqref{eq:app:m_expansion} becomes
\[
m_0^{\ell-1} >\!\!> d^{-\delta+1}m_0 \implies (\ell-1)\log_d m_0 \ge -\delta + 1 + \log_d m_0 \implies \delta \ge 1 + (2-\ell)\log_d m_0
\]
Plugging in \(m_0 = \sfrac{1}{\sqrt{d}}\)  we finally get \(\delta \ge \sfrac{\ell}{2}\), where the equality is the best possible value of the learning rate in order to make the escaping faster. Note that for \(\ell=1\) we are also bounded from Lemma~\ref{lem:app:boundongamma}, so \(\delta\ge1\). Combining with \eqref{eq:app:escapingtimes}, finally gives us the the minimal number of steps needed
\begin{equation}
    T \sim \begin{cases}
        d & \ell =1 \\
        d\log d & \ell=2 \\
        d^{\ell-1} & \ell\ge3 \\
 \end{cases}.
\end{equation}
This result matches \cite{arous2021online}.
\subsection{Extension to $p>1$}
The extension to general two-layer network student functions ($p>1$), while keeping always the target fixed to be a single-index one, can be readily done by performing an analysis similar to the above. The considerations on the weak recovery trade-offs done in the previous sections are not changed upon re-scaling the learning rate with the hidden layer size $p=O(1)$, i.e. $\sfrac{\gamma_{\rm 2LNN}}{p} = \gamma_{\rm GLM}$. Therefore, the scaling laws detailed in the phase diagram (Fig.~\ref{fig:cold_start_pd}) are not modified, and only prefactors, i.e. quantity not scaling with the input dimension, change with respect to the $p=1$ case. We illustrate this phenomenon numerically in Fig.~\ref{fig:app:k1p4singleindex}. We leave the detailed theoretical analysis of the $p>1$ case for future work, with particular attention to the limit $p \to \infty$ which we believe is an interesting avenue of future research.

\begin{figure}
    \centering
    \includegraphics[width=.49\textwidth]{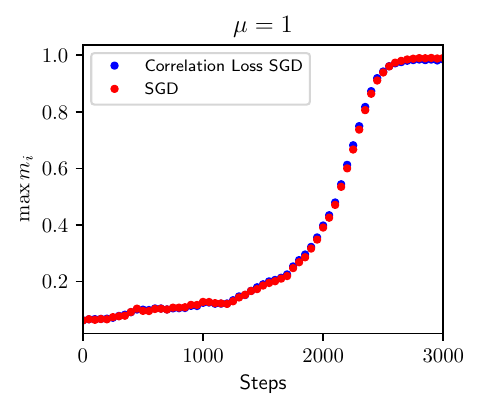}
    \includegraphics[width=.49\textwidth]{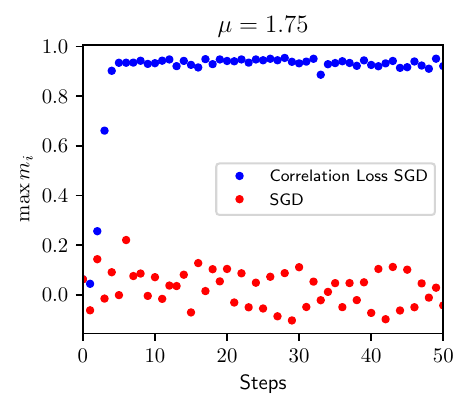}
    \caption{learning single-index teacher with a wide student, when information exponent is \(\ell=3\): \(f^\star(\vec{x})=\text{He}_3(\vec{w}^\star\cdot\vec{x}), f(\vec{x})=\sfrac14\sum_{i=1}^4\text{He}_3(\vec{w}_i\cdot\vec{x})\).
    Our theory extends to this case, showing that when \(\mu>\sfrac\ell2\) only correlation loss can weakly recover the target.
    (\(d=256, \gamma = \gamma_0\cdot p n_b d^{-\sfrac\ell/2}\))
    }
    \label{fig:app:k1p4singleindex}
\end{figure}
\section{Additional numerical investigation}
\label{sec:app:hyperparams}
In this appendix, we provide additional details on the numerical implementations presented in the main text, along with further explorations. The code to reproduce representative figures is available in \href{https://github.com/IdePHICS/batch-size-time-complexity-tradeoffs}{https://github.com/IdePHICS/batch-size-time-complexity-tradeoffs}.

\subsection{Cold start for multi-index models} \label{app:multiindex_model}
The theoretical considerations for weak recovery under cold starts presented in Theorems~\ref{thm:main:sgd_weak_recovery}\&\ref{thm:main:no_yhat_weak_recovery} are proven rigorously just for one-hidden neuron network learning single-index targets (\(p=k=1\)); this section aims to provide arguments to generalize this to the multi-index case. 

Note that for single-index models the initial saddle is the only critical point where the algorithm can get stuck during the dynamics, while this is not true in general for multi-index settings. Indeed, after having weakly recovered a subspace of the span of the target weights, the learning dynamics can encounter another saddle of the loss function; this behavior is known as saddle-to-saddle dynamics \cite{jacot2021saddle}. In this manuscript, we focus on escaping from the saddle at initialization, leaving further explorations of the dynamics to future work. We follow 
\cite{abbe2023sgd, dandi2023twolayer} where the authors generalize the concept of Information Exponent (defined in.~\eqref{def:main:information_exponent}) to the multi-index setting (See Definition $1$ of \cite{abbe2023sgd} and Definition $3$ of \cite{dandi2023twolayer}), let us call this quantity the \textit{Leap Index} of the target.
We expect that, as long as the dynamics around the saddle at initialization is analyzed, one can substitute the Information Exponent ($\ell$) of the teacher in the single-index phase diagram in Fig.~\ref{fig:cold_start_pd} with the Leap Index of the target. We explore the Time~/~Complexity tradeoffs in Figure~\ref{fig:app:multiindex} for a fixed teacher function with Leap Index equal to $3$: we observe a relevant decrease in the iterations needed to weakly recover the target subspace as the batch size is increased.  

\begin{figure}[t]
    \centering
    \includegraphics[width=.5\textwidth]{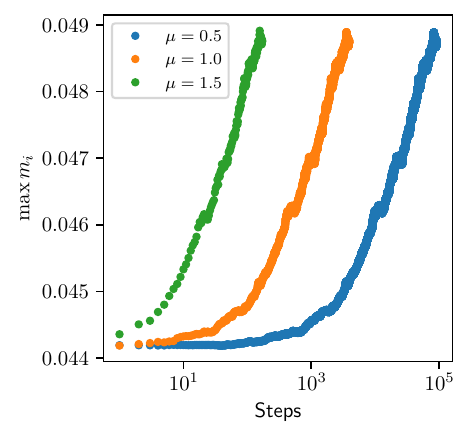}
    \caption{\textbf{{Multi-index large-batch benefits:}} Comparison between the performance of plain SGD learning multi-index model, for different values of \(\mu\). The target is \(h^\star(z_1,z_2,z_3)=\tanh{(z_1z_2z_3)}\), while the student is a wide 2-layer network \(f(\vec{x})=\frac1p\sum^p_{i=1}\tanh(\vec{w}_i\cdot\vec{x})\) (hence \(\ell=3,p=30, k=3, d=512\)). Using a larger batch speeds up the weak-correlation time even when the target is multi-index, and it is learned with a non-matching architecture.
    }
\label{fig:app:multiindex}
\end{figure}


\subsection{Behavior of \emph{Spherical SGD}}
In many theoretical work \cite{arous2021online,abbe2023sgd}, 
the algorithm used during training uses the \emph{spherical gradient} instead of the simple one. The update rule used instead of Equation~\eqref{eq:gd_update_weight} is
\begin{align}
    \vec w_{j,t+1} =  \frac{\vec{w}_{j,t} - \gamma \left(I_d-\vec{w}_{j,t}\vec{w}_{j,t}^\top\right)\nabla_{\vec w_{j,t}}\ell_t}{\norm{\vec{w}_{j,t} - \gamma \left(I_d-\vec{w}_{j,t}\vec{w}_{j,t}^\top\right)\nabla_{\vec w_{j,t}}\ell_t}} \qquad \forall t \in [T], \, \forall j \in [p]
    \label{eq:app:gd_update_weights}
\end{align}
In practice, only the gradient component orthogonal to the weights is taken into account. This algorithm is particularly convenient for theoretical analysis because it is easier to find a lower bound for the evolution of \(m_t\), since it is always true that 
\[
    \norm{\vec{w}_{j,t} - \gamma \left(I_d-\vec{w}_{j,t}\vec{w}_{j,t}^\top\right)\nabla_{\vec w_{j,t}}\ell_t} \ge 1,
\]
while its analogous for Projected SGD does not hold.

In this section, we want to show that \emph{Spherical SGD} is behaving like \emph{Correlation Loss SGD} when \(\gamma\) is not vanishing, namely that is possible to escape mediocrity when the batch size is sufficiently large. For small batch size, Projected SGD and Spherical SGD coincide, while when \(\gamma<0\) their behaviors are drastically different, and only the latter is able to escape mediocrity taking advantage of the large learning rate; a gap between the two is already noticeable at the \emph{Optimal Point}, where they both escape but the spherical is slightly faster.
Finally, note that is is possible to introduce a \emph{Correlation Loss Spherical SGD}, by changing the loss in the same way as the usual \emph{Correlation Loss SGD}. There is no practical difference between the two algorithms when working with correlation loss.

\begin{figure}[t]
    \centering
    \includegraphics[width=\textwidth]{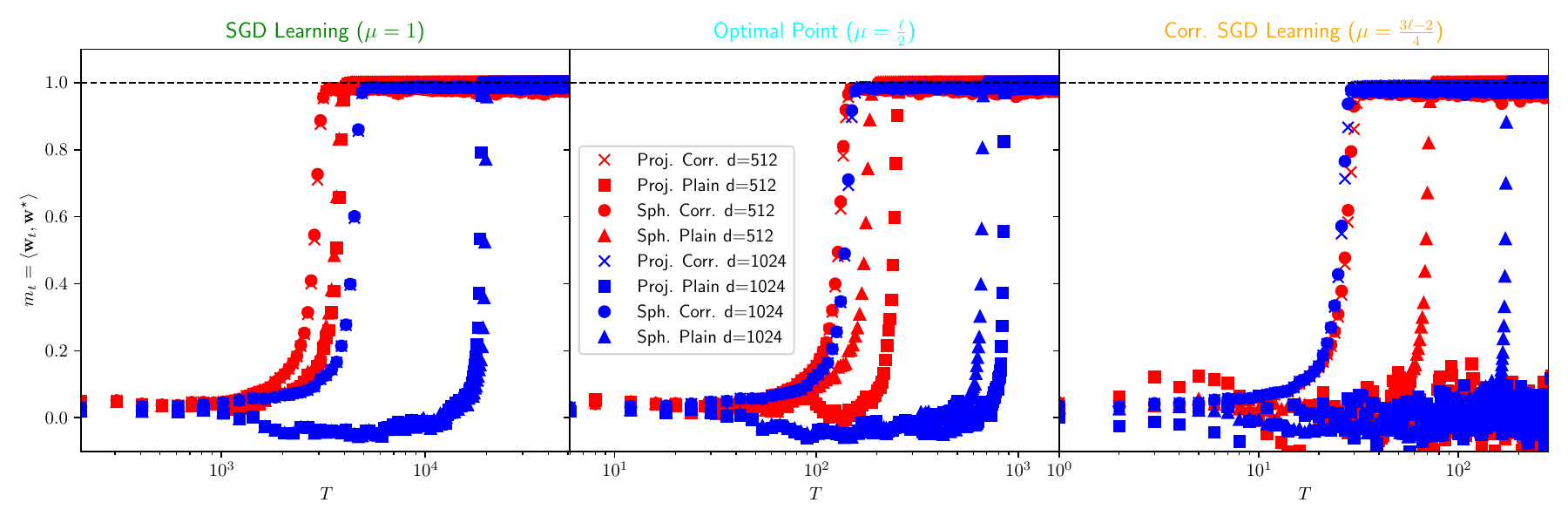}
    \caption{\textbf{{Correlation Loss SGD} weak recovery:} Comparison between the performance of plain SGD, the Correlation Loss SGD and Spherical SGD, in different regions of the phase diagram, and for different sizes \(d\).
    The plot shows the test error as a function of the optimization steps. Both the teacher and the student activation functions are fixed to \(\sigma = h^\star =\text{He}_3\), so the information exponent is \(\ell=3\). In all the three plots we vary the value of \(\mu\), while \(\delta = \mu - \sfrac{\ell}{2}\).
    Spherical SGD learns even in regions forbidden for plain SGD, as the Correlation Loss does. Note that the Spherical Correlation Loss is equivalent to the Projected Correlation Loss in all the regimes.
    }
\label{fig:app:spherical}
\end{figure}

\subsection{\emph{Adaptive SGD}: combining Correlation Loss SGD with plain SGD}

Despite these benefits, the correlation loss is not a good choice to fully learn the target. In this subsection, we explore the idea of combining the two algorithms to escape fast with correlation loss, and then reach the global minimum with the MSE loss.
We will call the combination of these two algorithms \emph{Adaptive SGD}.

We are going to test in the simplest case possible: GLM with \(\text{He}_3\) as activation function (we remark that there is no benefit in using Correlation Loss SGD over plain SGD when \(\ell\le2\)). If we run the algorithm for multi-index models, it would help to escape the initial saddle, but the algorithm may get stuck in another critical point that is not the global minimum. The study on how to escape fast from a critical point other than the initial one goes beyond the scope of this paper. Our \emph{Adaptive SGD} procedures works as follows:
\begin{enumerate}
    \item Make a Correlation Loss SGD step;
    \item If the Loss is smaller than 60\% of the initial loss, jump to Step~3, otherwise go back to Step~1;
    \item Reduce the learning rate of a factor \(0.995\) and do a Standard SGD step;
    \item If converged stop, otherwise go back to Step~3.
\end{enumerate}
The learning rate is progressively reduced because the plain SGD requires a lower learning rate compared to the one used by correlation loss to escape fast. Certainly, one can design a much more powerful algorithm than the one we present, but the goal here is just to show that the combination of the two is beneficial, and not to find the possible one.
\begin{figure}
    \centering
    \includegraphics{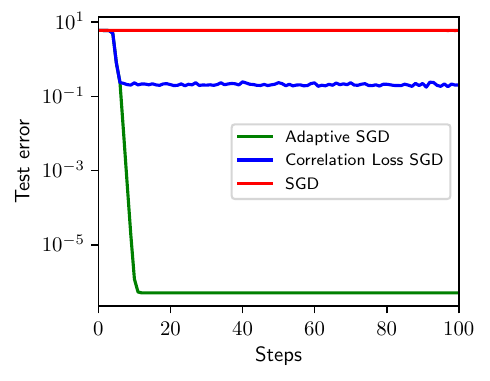}
    \caption{\textbf{Adaptive SGD:} The plot compares the performance of SGD and Correlation loss SGD, algorithms with Adaptive SGD; this protocol consists of first using correlation loss SGD to achieve weak recovery, and then switch to adaptive SGD for learning the target. (\(\ell=3,\mu=1.85, \delta=\mu-\frac\ell2, \Delta=10^{-6}\)).}
    \label{fig:app:adaptive}
\end{figure}
Figure~\ref{fig:app:adaptive} shows how the Adaptive SGD Algorithm is the best one when fully learning the target.

\subsection{Polylog regime example} \label{app:sec:polylog}
We showed in Section~\ref{sec:main:cold_start} that it is possible to push down the number of steps needed to weakly recover the target until it is growing less than any power law. In order to achieve this, we need to run \emph{Correlation Loss SGD} with \(n_b > d^{\ell-1}\) and \(\gamma > O(d^{1-\sfrac\ell2})\), as pictured by Figure~\ref{fig:app:deltamu_phasediagram}.
Proposition~\ref{prop:main:exact_asympt} shows that our theory for cold start, based on expansion of the process \eqref{eq:app:process}, is not valid, among other conditions, when \(\delta < 0\).
Therefore, the only case where we can simultaneously observe the \emph{polylog regime} and have an exact asymptotic description for the full dynamics is when \(\ell=2, \gamma = O(1)\) and \(n_b=O(d)\). Let's stick for simplicity with a GLM whose activation is \(\text{He}_2\). Note that the total sample complexity is always \(N = n_b T = O(d \log d)\).

\begin{figure}
    \centering
    \includegraphics[width=\textwidth]{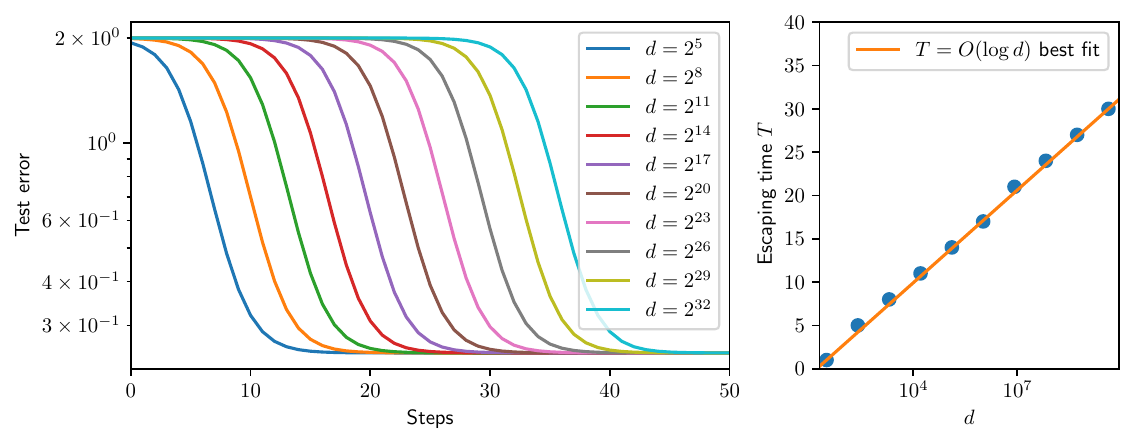}
    \caption{\textbf{Phase retrieval with large batch size:} Numerical integration of the process \eqref{eq:app:process}, for  \(f=f^\star=\text{He}_2, \gamma = O(1)\) and \(n_b=O(d)\). The escaping time dependence on the number of time steps is a \(T=O(\log d)\): we claim this to be valid for all the \emph{polylog regime}.}
    \label{fig:app:subpower}
\end{figure}
Figure~\ref{fig:app:subpower} shows a numerical test of our theory in this particular case. We see that as \(d\) grows, also the time needed to escape initial conditions grows. In the right part of the Figure we show that the exact dependence is \(T = O(\log d)\), that is indeed a polylog law.

\subsection{Large-batch corrections to asymptotic dynamics}
Although disappearing when taking the limit, the terms of evolution process \eqref{eq:app:process} coming from intra-batch correlation are useful for providing a better description at large but finite \(d\). Effectively, they are behaving as a first correction to the asymptotic limit.

In this section, we aim to provide numerical arguments about the importance of intra-batch correlations at finite \(d\). We stick with the GLM setting, with \(\erf\) as activation function. Note that since the information exponent of this target is 1, there is no mediocrity at initialization, we can set \(m=\vec w^\top \vec w ^\star = 0\) without falling in the \emph{cold start regime}.
\begin{figure}
    \centering
    \includegraphics[height=0.35\textwidth]{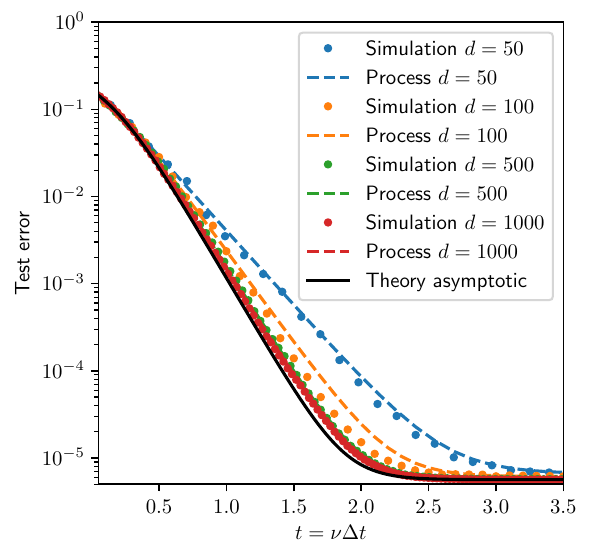}
    \includegraphics[height=0.35\textwidth]{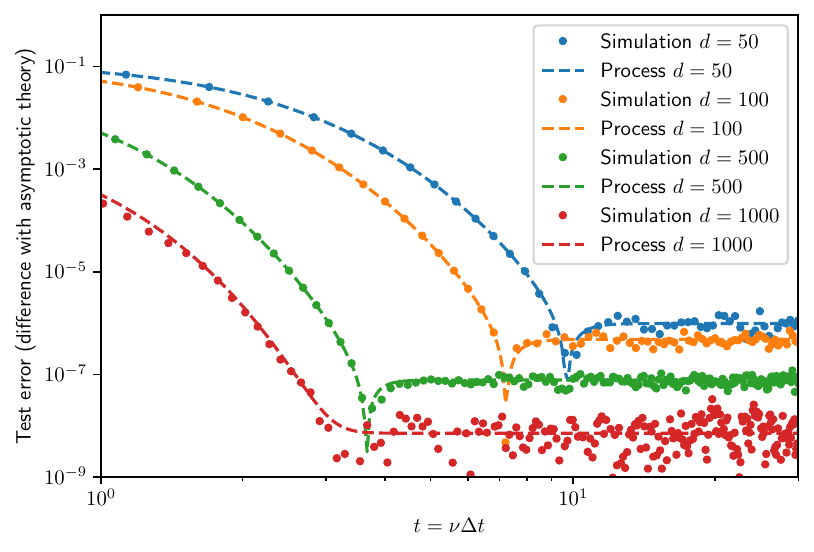}
    \caption{\textbf{Non asymptotic corrections:} Comparison between simulations (dots), exact asymptotic solution of ODE (continuous black line), and exact solution including the subleading large-batch corrections (dashed line). As expected, as \(d\to+\infty\) the simulations are getting closer and closer to the asymptotic solution; on the other hand, taking into account the batch correlations allows to have a better description of the dynamic even a small \(d\).}
    \label{fig:app:finitesizeasym}
\end{figure}
Figure~\ref{fig:app:finitesizeasym} shows simulations for different values of \(d\) (dots), accompanied by the full process dynamic that includes the intra-batch correlation terms (dashed lines); the asymptotic solution of the differential equations \eqref{eq:spherical_closed_form_ode} is the continuous black line. To enlighten the process even more, we also shows the difference between the asymptotic solution, at the actual finite \(d\) one on the right part of the figure. We see that the full process solution always match with the actual project SGD simulation; most importantly, when \(d\) grows the simulations are getting closer and closer to the asymptotic solution, confirming that the large batch plays no effect in high-dimension.

\end{document}